\documentclass[11pt]{article}

\usepackage{graphicx}
\usepackage{bbm}
\usepackage{hyperref,graphicx,fullpage,amsthm,amsmath,amsfonts,subcaption,amssymb,bm,url,epsfig,epsf,color,MnSymbol,mathbbol,fmtcount,caption,multirow,comment}
\usepackage{authblk}
\usepackage{color}
\usepackage[utf8]{inputenc} 
\usepackage[T1]{fontenc}    
\usepackage{booktabs}       
\usepackage{nicefrac}       
\usepackage{microtype}      
\usepackage{mathtools}
  
\usepackage[round]{natbib}   

\usepackage[noend]{algpseudocode}
\usepackage[linesnumbered,ruled]{algorithm2e}

\newcommand{\opnorm}[1]{\left\|#1\right\|}
\newcommand{\twonorm}[1]{\left\|#1\right\|_{2}}
\newcommand{\abs}[1]{\left|#1\right|}
\newcommand{\vct}[1]{\bm{#1}}
\newcommand{\mtx}[1]{\bm{#1}}

\newcommand{\relu}{\mathsf{ReLU}}

\newcommand{\C}{{\cal C}}

\newcommand{\R}{\mathbb{R}}

\newcommand{\Z}{\mathbb{Z}}

\newcommand{\E}{\mathbb{E}}
\newcommand{\eps}{\epsilon}

\newcommand{\pr}{\mathrm{Pr}}
\newcommand{\poly}{\mathrm{poly}}

\newcommand{\opt}{\mathsf{opt}}
\newcommand{\D}{\mathcal{D}}

\newcommand{\littlesum}{\mathop{\textstyle \sum}}

\newcommand{\bx}{\vct{x}}
\newcommand{\x}{\vct{x}}
\newcommand{\bu}{\vct{u}}
\newcommand{\bz}{\vct{z}}
\newcommand{\bw}{\vct{w}}
\newcommand{\bv}{\vct{v}}

\newcommand{\vol}{\mathsf{vol}}
\newcommand{\chow}{\mathsf{chow}}
\newcommand{\sur}{\mathsf{surr}}

\newcommand{\X}{\mathcal{X}}
\newcommand{\Y}{\mathcal{Y}}

\newtheorem{theorem}{Theorem}[section]

\newtheorem{remark}{Remark}

\newtheorem{question}{Question}
\newtheorem{lemma}[theorem]{Lemma}
\newtheorem{corollary}[theorem]{Corollary}

\newtheorem{definition}[theorem]{Definition}

\newtheorem{fact}{Fact}

\newcommand{\iprod}[1]{\langle#1\rangle}
\newcommand{\iprodtwo}[2]{\left\langle {#1}, {#2}\right\rangle}

\newcommand{\eqdef}{\stackrel{{\mathrm {\footnotesize def}}}{=}}

\makeatletter
\providecommand*{\boxast}{%
  \mathbin{
    \mathpalette\@boxit{*}%
  }%
}
\newcommand*{\@boxit}[2]{%
  \sbox0{$\m@th#1\Box$}%
  \ifx#1\displaystyle \ht0=\dimexpr\ht0+.05ex\relax \fi
  \ifx#1\textstyle \ht0=\dimexpr\ht0+.05ex\relax \fi
  \ifx#1\scriptstyle \ht0=\dimexpr\ht0+.04ex\relax \fi
  \ifx#1\scriptscriptstyle \ht0=\dimexpr\ht0+.065ex\relax \fi
  \sbox2{$#1\vcenter{}$}
  \rlap{%
    \hbox to \wd0{%
      \hfill
      \raisebox{%
        \dimexpr.5\dimexpr\ht0+\dp0\relax-\ht2\relax
      }{$\m@th#1#2$}%
      \hfill
    }%
  }%
  \Box
}
\makeatother

  \makeatletter
\def\BState{\State\hskip-\ALG@thistlm}
\makeatother

\allowdisplaybreaks


\title{Approximation Schemes for ReLU Regression}
\author[1]{Ilias Diakonikolas\footnote{ilias@cs.wisc.edu}}
\author[2]{Surbhi Goel\footnote{surbhi@cs.utexas.edu}}
\author[2]{Sushrut Karmalkar\footnote{susrutk@cs.utexas.edu}}
\author[2]{Adam R. Klivans\footnote{klivans@cs.utexas.edu}}
\author[3]{Mahdi Soltanolkotabi\footnote{soltanol@usc.edu}}
\affil[1]{University of Wisconsin, Madison}
\affil[2]{University of Texas at Austin}
\affil[3]{University of Southern California}
\date{}

\begin{document}

\maketitle

\begin{abstract}%
We consider the fundamental problem of ReLU regression, where the goal is to output the best fitting ReLU with respect to square loss given access to draws from some unknown distribution.  We give the first efficient, constant-factor approximation algorithm for this problem assuming the underlying distribution satisfies some weak concentration and anti-concentration conditions (and includes, for example, all log-concave distributions).  This solves the main open problem of Goel et al., who proved hardness results for any exact algorithm for ReLU regression (up to an additive $\epsilon$). Using more sophisticated techniques, we can improve our results and obtain a polynomial-time approximation scheme for any subgaussian distribution.  Given the aforementioned hardness results, these guarantees can not be substantially improved.
	
Our main insight is a new characterization of {\em surrogate losses} for nonconvex activations.  While prior work had established the existence of convex surrogates for monotone activations, we show that properties of the underlying distribution actually induce strong convexity for the loss, allowing us to relate the global minimum to the activation's {\em Chow parameters.}
\end{abstract}

\setcounter{page}{0}
\thispagestyle{empty}
\newpage

\section{Introduction} \label{sec:intro}
Finding the best-fitting ReLU with respect to square-loss -- also called ``ReLU Regression'' -- is a fundamental primitive in the theory of neural networks.  Many authors have recently studied the problem both in terms of finding algorithms that succeed under various assumptions and proving hardness results \citep{MR18,Mahdi17,GKK19-meurips,YS19,GKKT17, MR18}.  In this work, we consider the agnostic model of learning where {\em no assumptions} are made on the noise.



Recall the ReLU function $\relu_{\vct{w}}: \R^d \rightarrow \R$
parameterized by $\vct{w}$ is defined as 
$\relu_{\vct{w}}(\vct{x}) := \relu(\iprod{ \vct{w}, \vct{x} } ) = \max \left \{ 0, \iprod{ \bw, \vct{x} }  \right \}$ 
(for simplicity, let $\|\vct{w}\|_2 \leq 1)$. Given samples $(\vct{x}, y)$ drawn from a distribution $\D$ over $\R^d \times \R$, 
the objective of the learner is to find a hypothesis $h : \R^d \rightarrow \R$ that has square loss at most $\opt + \eps$, 
where $\opt \ll 1$ is defined to be the loss of the best fitting ReLU, i.e., 
\[
\opt := \min_{ \vct{w} \in \R^d } \E_{\D} \left[ \left( \relu(\iprod{ \vct{w}, \vct{x} } ) - y\right)^2  \right] \;.
\]
There are several hardness results known for this problem.
A recent result shows that finding a hypothesis achieving a loss of $O(\opt) + \eps$ is NP-hard when there are no distributional assumptions on $\D_\X$, the marginal of $\D$ on the examples~\citep{MR18}. 
Recent work due to \cite{GKK19-meurips} gives hardness results for achieving error $\opt + \eps$, 
even if the underlying distribution is the standard Gaussian.
This work also provides an algorithm that achieves error $O(\opt^{2/3}) + \epsilon$ 
under the assumption that $\D_\X$ is log-concave. 
The main problem open problem posed by \cite{GKK19-meurips} is the following:
\begin{question}
For the problem of ReLU regression, is it possible to recover a hypothesis achieving error 
of $O(\opt) + \eps$ in time $\poly(d, 1/\eps)$?
\end{question}
In this paper we answer this question in the affirmative. 
Specifically, we show that there is a fully polynomial time algorithm 
which can recover a vector $\vct{w}$ such that the loss of the corresponding ReLU function, 
$\relu_{\vct{w}}$, is at most $O(\opt) + \eps$. 
More formally, we prove the following:
\begin{theorem}\label{thm:informal_constfact}
If $\D_\X$ is isotropic log-concave, there is an algorithm that takes $\tilde{O}(d/ \epsilon^2)$ samples 
and runs in time $\tilde{O}(d^2 / \eps^2)$ and returns a vector $\vct{w}$ such that $\relu_{\vct{w}}$ 
has square loss $O(\opt) + \eps$ with high probability. 
\end{theorem}


The sample complexity of our algorithm is nearly linear in the problem dimension 
and hence information-theoretically optimal up to logarithmic factors. 
To establish this near-optimal sample complexity, we leverage intricate tools 
involving uniform {\em one-sided} concentration of empirical processes of log-concave distributions.

Additionally, we show that under stronger distributional assumptions and if the algorithm is allowed to be improper, 
i.e., if the hypothesis need not be the ReLU of a linear function, then it is possible to return a hypothesis 
that achieves a loss of $(1+\eta) \cdot \opt + \eps$ in polynomial time for any constant $\eta > 0$ as long as $\opt < 1$. 
\begin{theorem}\label{thm:informal_ptas}
If $\D_\X$ is $\nu$-subgaussian for $\nu \leq O(1)$,
then for any constant $\eta > 0$, there is an algorithm with sample complexity and running time 
$O \left( \frac{1}{\eps^2} \cdot \left( \frac{d}{\eta^3 \nu^2} \right)^{1/\eta^3} \right)$ that outputs a hypothesis 
$h : \R^d \rightarrow \R$ whose square loss is at most $(1+\eta) \cdot \opt + \eps$ with high probability. 
\end{theorem}
Given the hardness results of \cite{GKK19-meurips}, the aforementioned accuracy guarantees are essentially best-possible.  

\subsection{Our Approach} \label{ssec:approach}

A major barrier to minimizing the square loss for the ReLU regression problem is that it is nonconvex.  In such settings, gradient descent-based algorithms can potentially fail due to the presence of poor local minima.  In the case of ReLU regression, the number of these bad local minima for the square loss can be as large as exponential in the dimension \citep{AHW96}. 

Despite this fact, for well-structured noise models, it is possible to learn a ReLU with respect to square loss by applying results on isotonic regression \citep{KS09,KKSK11,KM17}. These results show that if the noise is bounded and has zero mean, it is possible to learn conditional mean functions of the form $\sigma_{\vct{w}}: \vct{x} \mapsto \sigma(\iprod{\vct{w}, \vct{x}})$ where $\sigma$ is a monotone and Lipschitz activation. This is proven via an analysis similar to that of the perceptron algorithm. 
It is not clear, however, how to extend these results to harder noise models. 

In retrospect, one way to interpret the algorithms from \cite{KS09} and \cite{KKSK11} 
is to view them as implicitly minimizing a {\em surrogate loss}\footnote{The analysis of \cite{KS09} and \cite{KKSK11} works directly with square loss and does not use the existence of a surrogate loss for its analysis.}.
The intuition is as follows: although a monotone and Lipschitz function need not be convex, 
it is not difficult to see that its {\em integral} is convex. This motivates the following definition of a surrogate loss:
\[
L^{\sur}_\D(\vct{w}) = \E_{(\vct{x}, y) \sim \D}\left[\int_0^{\iprod{ \vct{w}, \vct{x} } } (\sigma(a) - y) ~da\right] \;.
\]
Properties of this loss were explored early on in the work of \cite{AHW96} who gave a formal proof that the loss is convex (a succinct write-up of properties of this loss can also be found in notes due to \cite{K18}).  Thus, we can efficiently minimize this loss using gradient descent.  What is more subtle is the relationship of the minima of the surrogate loss to the minima of the original square-loss. 

The main insight of the current work is that algorithms that directly minimize this surrogate loss 
have strong noise-tolerance properties if the underlying marginal distribution satisfies some mild conditions. 
As a consequence, we prove that the GLMtron algorithm of \cite{KKSK11} (or equivalently projected gradient descent 
on the surrogate loss) achieves a constant-factor approximation for ReLU regression. 
The proof of this relies on three key structural observations:
\begin{itemize}
\item The first insight concerns the notion of the {\em Chow parameters} of a function. 
The Chow parameters $\chi_{\D}^f$ of a function $f:\R^d \rightarrow \R$ with respect to a distribution $\D$ 
are defined to be the first moments of $f$ with respect to $\D_\X$, i.e., 
$\chi_{\D}^f := \E_{\vct{x} \sim \D_\X}[ f(\iprod{\vct{w}, \vct{x}})\vct{x}]$. 
We show that the Chow parameters of a strictly monotone and Lipschitz activation function 
$\sigma$ {\em robustly characterize} the function, i.e., two functions with approximately the same Chow parameters have approximately the same loss. More precisely, 
any $\vct{w}$ that satisfies $\chi_{\D}^{\sigma_{\vct{w}}} = \E_{(\vct{x}, y) \sim \D}[y \cdot \vct{x}]$ induces 
a concept $\sigma_{\vct{w}}$ with square loss $O(\opt)$.

\item The second observation is that the gradient of the surrogate loss at $\vct{w}$ is the difference between 
the Chow parameters of $\sigma_{\vct{w}}$ and the first moments of the labels, 
$\chi_\D := \E_{(\vct{x}, y) \sim \D}[y \cdot \vct{x}]$, i.e., 
\[
\nabla_{\vct{w}} L^{\sur}_\D(\vct{w}) = \chi_{\D}^{\sigma_{\vct{w}}} - \chi_\D \;.
\]

\item The third insight is that if the underlying distribution $\D_{\X}$ satisfies some concentration and anti-concentration properties (satisfied, for instance, by log-concave distributions), 
then the surrogate loss is {\em strongly} convex. In particular, this holds for any activation that is strictly monotone and $1$-Lipschitz, including ReLUs. 
\end{itemize} 
Any strongly convex function achieves its minimum at a point where the gradient is zero. 
The first two observations now imply that the point where the surrogate loss has 
zero gradient corresponds to a weight vector achieving a loss of $O(\opt) + \epsilon$.


A naive analysis for the concentration of empirical gradients results in a sample complexity of roughly $O(d^4)$. 
To achieve the near-linear sample complexity of $O(d~ \mathrm{polylog}(d))$ 
in Theorem~\ref{thm:informal_constfact}, we show that while the gradient is not uniformly concentrated 
in all directions, it does concentrate from below in the direction going from the current estimate 
to the minimizer of the loss.


Theorem \ref{thm:informal_constfact} achieves a constant factor approximation to the ReLU regression problem 
when the underlying distribution is log-concave. It is not clear how to show that minimizing the surrogate loss 
alone can go beyond a constant factor approximation.
Still, it turns out that under a slightly stronger distributional assumption on $\D_\X$ (sub-gaussianity), 
we can give a polynomial-time approximation scheme (PTAS) for ReLU regression.

To achieve this, we build on the localization framework used to solve the problem of learning halfspaces 
under various noise models~\citep{Daniely15, ABL17}. The problem of learning halfspaces, however, 
differs from the problem of ReLU regression. One crucial difference is that for the problem of learning halfspaces, 
the agnostic noise model is equivalent to the noise model where an $\opt$ fraction of the labels are corrupted. 
In the case of ReLU regression, {\em every} point's label can potentially be corrupted. 

Our approach broadly proceeds in two stages:
\begin{itemize} 
\item First, we use our constant-factor approximation algorithm to recover a vector $\vct{w}$ 
satisfying $\| \vct{w} - \vct{w}^* \|_2^2 \leq O(\opt)$, 
where $\vct{w}^*$ is the vector achieving an error of $\opt$. We use this to partition the space 
into three regions for a certain choice of a parameter $t$. Our three regions are 
$T = \{\vct{u} \in \mathbb{R}^d: |\iprod{\vct{w}, \vct{u}}| \le t\}$, 
$T_{+} = \{\vct{u} \in \mathbb{R}^d: \iprod{\vct{w}, \vct{u}} > t\}$, 
and $T_{-} = \{\vct{u} \in \mathbb{R}^d: \iprod{\vct{w}, \vct{u}} < -t\}$. 

\item In each of these regions we find functions whose loss competes with that of the best fitting ReLU 
(i.e., $\relu(\iprod{\vct{w}^*, \vct{x}})$). 
\end{itemize}
Observe that $\relu_{\vct{w}^*}(\vct{x})$ takes the value $\iprod{\vct{w}^*, \vct{x}}$ for most of the region $T_+$. 
Intuitively, the best-fitting linear function $\vct{w}_{+}$ must achieve a loss comparable to 
$\relu_{\vct{w}^*}(\vct{x})$ for $T_+$. Similar reasoning shows that for the region $T_-$, $0$ is a good hypothesis. 
Using results from approximation theory, we show that the function $\relu_{\vct{w}^*}(\vct{x})$ in the region $T$ 
is closely approximated by a polynomial of degree $O\left( \frac{1}{\eta^3} \right)$. 
To find a function which achieves a comparable loss to the concept, we perform polynomial regression 
to find the best-fitting polynomial of appropriate degree in this region. 
Finally, our algorithm returns the following hypothesis $h$.  
\begin{align*}
h(\bx) = 
\begin{cases}
\iprod{\vct{w}_+, \vct{x}} \;,  & \vct{x} \in T_{+}\\
P(\vct{x}) \;, & \vct{x} \in T\\
0 \;, & \vct{x} \in T_{-}
\end{cases} \;.
\end{align*}
The paper by \cite{Daniely15} shows this result only for the uniform distribution on the sphere, while 
our result works for all sub-gaussian distributions. The analysis of this algorithm is nontrivial. 
In particular, in addition to using tools from approximation theory to derive the polynomial approximation, 
the choice of the parameter $t$ to partition our space is delicate, 
and we need to calculate approximations with respect to complicated marginal distributions 
that do not have nice closed-form expressions.

\subsection{Prior and Related Work}\label{ssec:related}

Here we provide an overview of the most relevant prior work.
\cite{GKKT17} give an efficient algorithm for ReLU regression that succeeds with respect to 
any distribution supported on the unit sphere, but has sample complexity and running time exponential in $1/\eps$.  
\cite{Mahdi17} shows that SGD efficiently learns a ReLU in the realizable setting 
when the underlying distribution is assumed to be the standard Gaussian. 
\cite{GKM18} gives a learning algorithm for one convolutional layer of ReLUs for any symmetric distribution (including Gaussians).  \cite{GKK19-meurips} gives an efficient algorithm for ReLU regression with error guarantee of $O(\opt^{2/3})+\eps$. 

\cite{YS19} shows that it is hard to learn a single ReLU activation via stochastic gradient descent, 
when the hypothesis used to learn the ReLU function is of the form $N(\bx) := \sum_{i=1}^r u_i f_i(\bx)$ and 
the functions $f_i(\bx)$ are random feature maps drawn from a fixed distribution. 
In particular, they show that any $N(\bx)$ which approximates $\relu(\iprod{\vct{w}^*, \vct{x}} + b)$ (where $\| \vct{w} \|_2 = d^2$ 
and $b \in \R$) up to a small constant square loss, must have one of the $|u_i|$ being exponentially large in $d$ for some $i$ 
or have exponentially many random features in the sum (i.e., $r \geq \exp(\Omega(d))$. Their paper makes the point that 
regression using {\em random features} cannot learn the ReLU function in polynomial time. 
Our results use different techniques to learn the unknown ReLU function that are not captured by this model.

We note that Chow parameters have been previously used in the context of learning halfspaces under well-behaved distributions, see, e.g.,~\cite{OS08short, DDFS:12stoc, DiakonikolasKM19} and references therein.
The technique of localization has been used extensively in the context of learning halfspaces over various structured 
distributions. Specifically, \cite{ABL17} use this technique to learn origin-centered 
halfspaces with respect to log-concave distributions in the presence of agnostic noise, 
obtaining an error guarantee of $O(\opt) + \eps$. Subsequently, 
\cite{Daniely15} uses an adaptation of the localization technique in conjunction 
with the polynomial approximation technique from \cite{KKM+:05} to obtain a PTAS 
for the problem of agnostically learning origin-centered halfspaces under the uniform distribution over the sphere. 
More recently, \cite{DKS18-nasty} obtain similar guarantees in the presence of nasty noise, where the 
halfspace need not be origin-centered. 

While the problem of learning halfspaces is related to that of ReLU regression, 
we stress that for ReLU regression {\em every} label may be corrupted (possibly by arbitrarily large values), 
while in the context of learning halfspaces only an $\opt$ fraction of the labels are corrupted. 
This is because the loss for halfspace learning is $0/1$ instead of the square-loss. 
Indeed, a black-box application of the results for halfspace learning in the context ReLU regression 
results in the suboptimal guarantee of $O(\opt^{2/3})$ \citep{GKK19-meurips}. 

\section{Preliminaries} \label{sec:prelims}
\paragraph{Notation.} For $n \in \Z_+$, we denote $[n] \eqdef \{1, \ldots, n\}$.
We will use small boldface characters for vectors. 
For $\vct{x} \in \R^d$, 
and $i \in [d]$, $\vct{x}_i$ denotes the $i$-th coordinate of $\vct{x}$, and 
$\|\vct{x}\|_2 \eqdef (\littlesum_{i=1}^d \vct{x}_i^2)^{1/2}$ denotes the $\ell_2$-norm
of $\vct{x}$. We will use $\langle \vct{x}, \vct{y} \rangle$ for the inner product between $\vct{x}, \vct{y} \in \R^d$. 
We will use $\E[X]$ for the expectation of random variable $X$ and $\Pr[\mathcal{E}]$
for the probability of event $\mathcal{E}$. For two functions $f, g$ let $f \lesssim g$ mean that there exists a $C>0$ such that $f(x) \leq C g(x)$ for all $x > C$ and $f \gtrsim g$ denote $g \lesssim f$. $B(d,W)$ denotes the $d$-dimensional Euclidean ball at the origin with radius $W$, that is, $B(d,W):= \{\vct{x} \in \R^d ~|~ \|\vct{x}\|_2 \le W\}$. We say $f = O(g)$ if $f \lesssim g$, also we use $\tilde{O}$ to hide log factors of the input. We will use $\sigma^{\prime}(\bx)$ to denote a subgradient of $\sigma$ at the point $\bx$. 


\paragraph{Learning Models.} We start by reviewing the PAC learning model~\cite{Vapnik82, val84}.
Let $\C$ be the target (concept) class of functions $f: \X \to \mathcal{Y}$, 
$\mathcal{H}$ be a hypothesis class, and
$\ell: \mathcal{H} \times \X \times\Y \to \R$ be a loss function.
In the (distribution-specific) agnostic PAC model~\cite{Haussler:92, KSS:94}, 
we are given a multi-set of labeled examples $(\vct{x}^{(i)}, y^{(i)})$
that are i.i.d.~samples drawn from a distribution $\D = (\D_{\X}, \D_\Y)$ on $\X \times\Y$,
where $\vct{x}^{(i)} \sim \D_{\vct{x}}$. The marginal distribution $\D_{\vct{x}}$ is assumed to lie 
in a family of well-behaved distributions.
The goal is to find a hypothesis $h \in \mathcal{H}$ that approximately minimizes 
the expected loss $L_{\D}(h) := \E_{(\vct{x}, y) \sim \D} [\ell(h(\vct{x}), y)]$, compared to
$\opt_{\D}(\C)  := \min_{f \in \C} L_{\D}(f)$. In this paper, we will have $\X = \R^d$, $\Y =\R$, 
and $\ell(h(\vct{x}), y) = (h(\vct{x}) - y)^2$.
We will focus on constant factor approximation algorithms, that is, we will want a hypothesis which 
satisfies $L_{\D}(h) \le C \cdot \opt_{\D}(\C) + \epsilon$ for some universal constant $C > 1$ and $\epsilon \in (0,1)$. 
If the hypothesis $h \in \C$ then the learner is {\em proper} else it is called {\em improper}.


\paragraph{Problem Setup.} We consider the concept class of Generalized Linear Models (GLMs) $\C_\sigma:= \{\vct{x} \rightarrow \sigma(\langle \vct{w}, \vct{x} \rangle)\}$ for activation functions $\sigma: \R \rightarrow \R$ which are non-decreasing and $1$-Lipschitz. Common activations such as ReLU and Sigmoid satisfy this assumption.
We use the $L_2$-error as our loss function, i.e., $L_{\D}(h) := \E_{(\vct{x}, y) \sim \D} [(h(\vct{x})-y)^2]$. 
We overload the definition by setting $L_{\D}(f,g) := \E_{(\vct{x}, y) \sim \D} [(f(\vct{x})-g(\vct{x}))^2]$.  
Our goal is to design a proper constant-approximation PAC learner for class $\C_\sigma$ in time 
and sample complexity polynomial in the input parameters. 

In this paper, we focus primarily on the ReLU activation, that is, $\relu(a) = \max(0, ~a)$.  We also restrict ourselves to isotropic distributions, that is, $\E_{\vct{x} \sim \D_\X}[\vct{x}] = 0$ and $\E_{\vct{x} \sim \D_\X}[\vct{x}\vct{x}^T] = \mtx{I}$. We also assume that the labels are bounded in absolute value by 1 for ease of presentation. For approximate learning guarantees, our results go through if we assume the distribution of labels is sub-exponential.


\begin{definition}[Chow parameters]
Given a distribution $\D$ over $\R^d \times \R$, for any function $f: \R^d \rightarrow \R$, 
define the (degree-$1$) Chow parameters of $f$ w.r.t. $\D$ as $\chi_{\D}^f:= \E_{\vct{x} \sim \D_{\vct{x}}}[f(\vct{x})\vct{x}]$.
\end{definition}
For a sample $S$ drawn from $\D$, we also define the corresponding empirical Chow parameter with respect to $S$ 
as $\widehat{\chi}_S^f := \frac{1}{|S|} \sum_{(\vct{x}, y) \in S} f(\vct{x})\vct{x}$. 

We overload notation by defining the true Chow parameters as $\chi_\D = \E_{(\vct{x}, y) \sim \D}[y \vct{x}]$ 
and its corresponding empirical true Chow parameter w.r.t. $S$ as 
$\widehat{\chi}_S := \frac{1}{|S|} \sum\limits_{(\vct{x}, y) \in S} y\vct{x}$.

\begin{definition}[Chow distance]
Given distribution $\D$ over $\R^d \times \R$, for any functions $f,g: \R^d \rightarrow \R$, 
define the Chow distance between $f$ and $g$ w.r.t. $\D$ as $\chow_{\D}(f, g) = \|\chi_{\D}^f - \chi_{\D}^g\|_2$, 
that is, the Euclidean distance between the corresponding Chow parameters.
\end{definition}

\begin{lemma}[Chow distance to function distance]\label{lem:chow_upper} Let $\D$ be such that the marginal on $\X$ is isotropic. For any functions $f$ and $g$, $\left\|\chi_\D^{f} - \chi_\D^{g}\right\|_2 \le \sqrt{L_\D(f, g)}$.
\end{lemma}
\begin{proof}
	We have
	\begin{align*}
	\|\chi_\D^{f} - \chi_\D^{g}\|_2 &= \|\E_{(\vct{x}, y) \sim \D}[(f(\vct{x}) - g(\vct{x}))\vct{x}]\|_2\\
	&= \max_{\|\vct{u}\|_2 = 1} \E_{\D}[(f(\vct{x}) -g(\vct{x})) \iprod{\vct{u}, \vct{x}}]\\
	&\le \sqrt{L_\D(f, g)}\max_{\|\vct{u}\|_2 \le 1} \sqrt{\E_{\D}[ \iprod{\vct{u}, \vct{x}}^2]} = \sqrt{L_\D(f, g)}.
	\end{align*}
Here the first equality follows from the variational form of the Euclidean norm and the last follows from applying Cauchy-Schwartz inequality and using isotropy of the underlying distribution on $\X$.
\end{proof}

\begin{corollary}[Chow-distance from true Chow vector]\label{lem:chow_opt} Let $\D$ be such that the marginal on $\X$ is isotropic. For {\emph any} activation function $\sigma : \mathbb{R} \rightarrow \mathbb{R}$ and vector $\vct{w} \in \R^d$, we have $\|\chi_\D - \chi_{D}^{\sigma_{\vct{w}}}\|_2 \le \sqrt{L_\D(\sigma_{\vct{w}})}$.
\end{corollary}
\begin{proof}
	Letting $f = \E[y|\vct{x}]$ and $g = \sigma_{\vct{w}}$ in Lemma \ref{lem:chow_upper} gives us,
	\begin{align*}
	\|\chi_\D - \chi_{D}^{\sigma_{\vct{w}}}\|_2^2 &\le \E_{\D}[(\E[y|\vct{x}] - \sigma_{\vct{w}}(\vct{x}))^2]\\
	&\le \E_{\D}[(y - \sigma_{\vct{w}}(\vct{x}))^2] = L_\D(\sigma_{\vct{w}}).
	\end{align*}
	Here the last inequality follows from an application of Jensen's inequality.
\end{proof}

\paragraph{Organization.} In Section 3, we give an algorithm to find a weight vector 
that matches the true Chow parameters for the class of GLMs. In Section 4, 
we show that under certain assumptions on the activation function, the so obtained weight vector in fact 
gives us the approximate learning guarantee. In Section 5, we show that, for isotropic log-concave distributions, 
the ReLU satisfies our assumptions, and combining the previous techniques gives us the desired approximate 
learning result. Finally, in Section 6 we give an algorithm that improves the approximation factor
to $1 + \eta$ for any constant $0 < \eta \le 1$ at the cost of improper learning.


\section{Matching Chow Parameters via Projected Gradient Descent} \label{sec:alg}




In this section, we show that projected gradient descent on the surrogate loss
outputs a hypothesis $\sigma_{\vct{w}}$ whose Chow parameters nearly match the true Chow parameters, $\E[y\vct{x}]$. 
More formally, we redefine the surrogate loss as follows:
\[
L^{\sur}_\D(\vct{w}) = \E_{(\vct{x}, y) \sim \D}\left[\int_0^{\iprod{ \vct{w}, \vct{x} } } (\sigma(a) - y) ~da\right] = \E_{(\vct{x}, y) \sim \D}\left[\widetilde{\sigma}(\langle \vct{w}, \vct{x} \rangle) - y \langle \vct{w}, \vct{x} \rangle\right] \;.
\]
Here $\widetilde{\sigma}$ is the anti-derivative of $\sigma$. 
For example, for the ReLU activation, we have that $\widetilde{\relu}(a) = 0$ for all $a \le 0$ and 
$\widetilde{\relu}(a) = a^2/2$ otherwise. 
We correspondingly define the empirical version of the surrogate loss over sample set $S$ as $\hat{L}_S^\sur$.

We note that the gradient of $L^{\sur}_\D$ is directly related to the Chow parameters as follows
\[
\nabla L^{\sur}_\D(\vct{w}) = \E[\sigma(\langle \vct{w}, \vct{x} \rangle)\vct{x}] - \chi_\D = \chi_\D^{\sigma_{\vct{w}}} - \chi_{\D} \;.
\]
Furthermore, the Hessian can be computed as
\[
\nabla^2 L^{\sur}_\D(\vct{w}) = \E[\sigma^\prime(\langle \vct{w} , \vct{x} \rangle) \vct{x}\vct{x}^T] \succcurlyeq 0 \;.
\]
Where $\sigma^{\prime}$ is a subgradient. Here the last inequality follows from the non-decreasing property of $\sigma$. 
Thus, we have that $L^{\sur}_\D$ is convex. 
Moreover, since $\sigma$ is $1$-Lipschitz, and our distribution is isotropic, we have that 
$1 \succcurlyeq \nabla^2 L^{\sur}_\D(\vct{w})$ implying that $L^{\sur}_\D$ is $1$-smooth. 
Since minimizing the surrogate loss minimizes the gradient norm of the loss, 
loss minimization matches the Chow parameters of the GLM to the true Chow parameters.

\begin{algorithm}
\SetKwInOut{Parameter}{Parameter}
	\caption{Projected Gradient Descent on Surrogate Loss}\label{algo:pgd}
	\KwIn{Set $S = (\vct{x}^{(i)}, y^{(i)})_{i=1}^m$ i.i.d. samples drawn from $\D$}
	\Parameter{Learning rate $\eta > 0$ and weight bound $W$}
		 $\vct{w}^{(0)} := 0^d$\\
		\For{$t = 0, \ldots, T$}{
		$\vct{v}^{(t+1)} := \vct{w}^{(t)} - \eta \nabla \widehat{L}^\sur_S(\vct{w}^{(t)})$\\
		$\vct{w}^{(t+1)} := \Pi_{B(d, W)}(\vct{v}^{(t+1)})$ (Projection step)
		}
\end{algorithm}

By standard Projected Gradient Descent analysis with approximate gradients, we have the following theorem, the proof of which is in Section~\ref{sec:nosc} of the appendix. 

\begin{theorem} \label{thm:nosc}
	Suppose $m$ is sufficiently large so that for all $\vct{w} \in B(d, W)$ we have
	\[
	\| \nabla L^\sur_\D(\vct{w}) - \nabla \hat{L}^\sur_S(\vct{w}) \|_2 \le \epsilon \;.
	\]
	Also suppose that the minimizer of $L_\D^\sur$ lies in $B(d, W)$. 
	Then Algorithm \ref{algo:pgd} when run on $m$ samples from $\D$ with weight bound $W$ and $\eta < 1/4$ for $T \ge \frac{32 W^2}{4\epsilon W + \epsilon^2}$ iterations has an iteration $T' \le T$ such that
	\[
	\| \chi_\D^{\sigma_{\vct{w}^{(T')}}} - \chi_\D\|_2^2 \le 8\epsilon W + 2\epsilon^2.
	\]
	\end{theorem}
	
Subsequently, we can use a fresh batch of samples and choose the hypothesis with the smallest gradient. 
Assuming our distribution satisfies certain concentration properties, we can bound the number 
of samples needed by the above algorithm using the following lemma whose proof we defer to Section~\ref{sec:concentration_additive} of the appendix. 

	
	\begin{lemma} \label{lem:concentration_additive}
	If $\D$ is a distribution such that for every $\bv$, $\iprod{\vct{x}, \vct{v}}$ has a density bounded above by $\exp(-\iprod{\vct{x}, \vct{v}}^t)$ for some $t>0$, then for 
	$m  \geq \Omega \left( \left( W\frac{d}{\epsilon} \log \frac{W}{\epsilon} \log \frac{1}{\delta} \right)^{2/t} \right)$, 
	for all $\vct{w} \in B(d, W)$ we have that
	\[
	\pr_{S \sim \D^m} \left[\left \| \nabla L^\sur_\D(\vct{w}) - \nabla \hat{L}^\sur_S(\vct{w}) \right \|_2 \le \epsilon  \right] \ge 1 - \delta. 
	\]

\end{lemma}

\paragraph{Faster Rates under Strong Convexity} 
If we assume that $L_\D^\sur$ is strongly convex and restrict to a bounded fourth moment distribution, we can get much faster rates and improved sample complexity (in fact linear in the dimension $d$ up to log factors).

\begin{definition}[Strong-Convexity]\label{def:sc} 
We say that the activation $\sigma$ satisfies $\mu$-strong convexity 
w.r.t. distribution $\D$, if for all $\vct{u}, \vct{v}$ there exists $\mu > 0$ such that
\[
 \langle \chi_\D^{\sigma_{\vct{u}}} - \chi_\D^{\sigma_{\vct{v}}}, \vct{u} - \vct{v} \rangle \ge \mu \| \vct{u} - \vct{v}\|_2^2 \;.
\]
\end{definition}

\begin{theorem} \label{thm:mahdi}
Let $\D$ be such that $\D_{\X}$ is isotropic log-concave. Suppose that the minimizer of $L_\D^\sur$ lies in $B(d, W)$. If $\sigma$ satisfies $\mu$-strong convexity w.r.t. $\D$ then for Algorithm \ref{algo:pgd} (without the projection step) run with $\eta \le 1/16$ and
\begin{align*}
m \geq \tilde{\Omega}\left( \frac{(\mu+1)}{\mu^2 \epsilon^2}d \log^4 \left(\frac{d}{\delta}\right)\left(W+1\right)^2+ \frac{d}{\mu^2} \log\left(\frac{W + 1}{\mu \delta}\right) \right) \quad\text{where}\quad 0\le \epsilon\le W
\end{align*}
 after $T \ge \frac{2\log\left(\frac{9 W}{\epsilon}\right)}{\log\left(1-\frac{\mu\eta}{6}\right)}$ iterations, $
	\| \chi_\D^{\sigma_{\vct{w}^{(T)}}} - \chi_\D\|_2 \le \epsilon$ holds with probability at least $1 - \delta$ as long as $\delta \ge e^{-O(\sqrt{d})}$.
\end{theorem}
The proof of Theorem~\ref{thm:mahdi} is deferred to Section~\ref{sec:mahdi} in the Appendix.

\section{Matching Chow Parameters Suffices for Approximate Learning}
In this section, we show that under certain assumptions on the activation function, matching Chow vectors implies small loss of the surrogate minimizer. We subsequently show that commonly used activation functions such as 
$\relu$ satisfy this assumption.

\begin{definition}[Chow Learnability]\label{def:cl}
We say that an activation function satisfies $\beta$-{\em Chow Learnability} w.r.t. some distribution $\D$ if for all $\vct{u}, \vct{v}\in \R^d$ and some fixed constant $\beta > 0$, we have that
\[
L_\D(\sigma_{\vct{u}}, \sigma_{\vct{v}}) \le \beta \cdot\| \chi_\D^{\sigma_{\vct{u}}} - \chi_\D^{\sigma_{\vct{v}}}\|_2^2 \;.
\]
\end{definition}
We will require the following lemma, proved in Section~\ref{sec:sccl}. 

\begin{lemma} \label{lem:sccl} If a $1$-Lipschitz activation $\sigma$ satisfies 
$\mu$-strong convexity w.r.t. $\D$ such that $\D_\X$ is isotropic, then the activation also satisfies $\mu$-Chow Learnability. 
\end{lemma}

\begin{remark}
Observe that Chow learnability may be a much weaker notion than strong convexity, since strong convexity requires parameter closeness. For activations with bounded ranges, such as sigmoid, it is possible for the 
loss to be small and Chow parameters to be close while the vectors themselves may be far.
\end{remark}

If the activation satisfies the Chow learnability condition, then we can show that a hypothesis nearly matching the Chow parameters attains small loss.

\begin{theorem} \label{thm:chowtoopt}
Let $\sigma$ be such that it satisfies $\beta$-{\em Chow Learnability} 
w.r.t. $\D$ with $\D_\X$ being isotropic. Suppose $\vct{w}$ is such that 
$\|\chi_\D^{\sigma_{\vct{w}}} - \chi_\D\|_2^2 \le \epsilon$. Then we have
\[
L_\D(\sigma_{\vct{w}}) \le 2~\opt_\D(C_\sigma)\left(1 + 2\beta\right) + 4\beta \epsilon \;.
\]
\end{theorem}
\begin{proof}
Let $\sigma_{\vct{w}^*}$ be the function attaining the loss $\opt_\D(C_\sigma)$. By assumption on $\sigma$, we have
\begin{align*}
L_\D(\sigma_{\vct{w}}, \sigma_{\vct{w}^*}) &\le \beta \cdot\| \chi_\D^{\sigma_{\vct{w}}} - \chi_\D^{\sigma_{\vct{w}^*}}\|_2^2\\
&\le 2~\beta ~ \left(\|\chi_\D^{\sigma_{\vct{w}}} - \chi_\D\|_2^2 +  \| \chi_\D^{\sigma_{\vct{w}^*}} - \chi_\D\|_2^2 \right) \\
&\le 2 ~\beta ~\left(\epsilon + \opt_\D(C_\sigma)\right).
\end{align*}{}
Here the last inequality follows by Corollary \ref{lem:chow_opt}. Also using triangle inequality,
\[
L_\D(\sigma_{\vct{w}}) \le 2~\opt_\D(C_\sigma) + 2 ~L_\D(\sigma_{\vct{w}}, \sigma_{\vct{w}^*}).
\]
Combining the above gives us the desired result.
\end{proof}
\begin{remark}
In the above guarantee, we can replace $\opt_\D(C_\sigma)$ by $\min_{c \in C_\sigma} \E[(\E[y|\vct{x}] - c(\vct{x}))^2]$ (see proof of Lemma \ref{lem:chow_opt}). In the p-concept setting, where $\E[y|\vct{x}] = c^*(\vct{x})$ this is potentially a tighter guarantee. This is because $\min_{c \in C_\sigma} \E[(\E[y|\vct{x}] - c(\vct{x}))^2]$ is in fact 0 whereas $\opt_\D(C_\sigma)$ might be large. Since we are focused on the agnostic setting, we will stick to using $\opt_\D(C_\sigma)$ in our results.
\end{remark}

\section{Constant Factor Approximation for ReLU Regression}

In this section, we present a constant factor approximation algorithm 
for ReLU regression over any isotropic log-concave distribution 
using the techniques developed in the previous sections.

\begin{theorem}\label{thm:mainReLU}
Let $\D$ be such that $\D_\X$ is isotropic log-concave and assume the labels are bounded. 
Let $\relu_{\vct{w}^*}$ achieve loss $\opt_\D(C_\relu)$ and assume that 
$\|\vct{w}^*\|_2\le W_{\opt}$. Then Algorithm~\ref{algo:pgd} outputs a vector $\vct{w}$ such that
\[
L_\D(\relu_{\vct{w}}) \le O\left(\opt_{\D}(C_\relu)\right) + \eps \;,
\]
with probability $1-\delta$ using $m\gtrsim \frac{d}{\epsilon^2} \log^4 \left(\frac{d}{\delta}\right)\left(W_{\opt}+1\right)^2$ 
samples, for $0\le \epsilon\le W_{\opt}$, and $O\left(d m\log\left(\frac{W}{\epsilon}\right)\right)$ time.
\end{theorem}

Our main observation is that the ReLU activation 
satisfies the strong convexity condition w.r.t. any isotropic log-concave distribution.

\begin{lemma}[Strong Convexity of ReLU]\label{lem:relusc} 
Let $\D$ be such that $\D_\X$ is isotropic log-concave. 
Then there exists some fixed constant $\mu > 0$ such that ReLU is $\mu$-strongly convex w.r.t. $\D$.
\end{lemma}
	
\paragraph{Proof Sketch. } Since the ReLU is $1$-Lipschitz and non-decreasing, we have
\begin{align*}
		(\chi_\D^{\relu_{\vct{v}}} - \chi_\D^{\relu_{\vct{u}}} )^T(\vct{v} - \vct{u})
		&= \E\left[\left(\relu(\langle \vct{v}, \vct{x} \rangle) -\relu(\langle \vct{u}, \vct{x} \rangle)\right) ((\vct{v} - \vct{u}) \cdot x)\right]\\
		&\ge \E\left[\left(\relu(\langle \vct{v}, \vct{x} \rangle) -\relu(\langle \vct{u}, \vct{x} \rangle)\right)^2 \right] \;.
		\end{align*}
Now our goal is to bound from below the error between the two ReLUs by the distance between the corresponding vectors. 
Due to the anti-concentration properties of log-concave distributions, there is sufficient probability mass 
in a constant radius ball around the origin. This enables us to exploit the linear region of the corresponding ReLUs 
to establish the lower bound. We defer the full proof to Section~\ref{sec:relusc} in the Appendix. 

\paragraph{Proof of Theorem \ref{thm:mainReLU}}
By Lemma \ref{lem:relusc}, the ReLU activation satisfies $\mu$-strong convexity w.r.t. $\D$ for some constant $\mu > 0$. 
This implies that $L_\D^\sur$ is strongly-convex and therefore the minimizer of $L_\D^\sur$ (say $\vct{w}$) 
satisfies $\chi_\D = \chi_\D^{\relu_{\vct{w}}}$. Using Lemma \ref{lem:chow_opt} and the strong convexity of ReLU, 
we have that
\begin{align*}
\|\vct{w}^* - \vct{w}\|_2 &\lesssim \| \chi_\D^{\relu_{\vct{w}^*}} - \chi_\D^{\relu_{\vct{w}}}\|_2\\
&= \| \chi_\D^{\relu_{\vct{w}^*}} - \chi_\D\|_2\\
&\le \sqrt{L_\D(\relu_{\vct{w}^*})} = \sqrt{\opt_\D(C_\relu)} \;. 
\end{align*}
Therefore, $\|\vct{w}\|_2 \le W_{\opt} + O\left(\sqrt{\opt_\D(C_\relu)}\right)$. 
It is not hard to see that with bounded labels $\opt_\D(C_\relu) \le O({W_{\opt}}^2 +1))$. 
Therefore, we can now apply Theorem \ref{thm:mahdi} to find a hypothesis with Chow distance at most $\eps$. 
The result now follows directly from Theorem \ref{thm:chowtoopt}.

\section{A PTAS for ReLU Regression} \label{sec:ptas}
In this section, we show that if the activation is the $\relu$ function 
we can solve the problem of finding the best fitting $\relu$ up to a $(1 + \eta)$-approximation, 
when the underlying marginal over the input is sub-gaussian. 
We assume that $\opt:= \opt_\D(C^1_\relu) \leq c \leq 1$, for some constant $c$. 

We define sub-gaussian distributions here:

\begin{definition}
A distribution $\D$ on $\R^d$ is called $\nu$-\emph{subgaussian}, $\nu>0$, if for any direction $\bv$ 
the probability density function of 
$\iprod{\bx, \bv}$ where $\bx \sim \D$, $p_{\bv}(\bx)$ satisfies 
$p_{\bv} (\bx) = O \left( \frac{1}{\nu} \cdot \exp\left( - \frac{(\bv \cdot \bx)^2}{2 \nu^2}\right) \right)$.
\end{definition}

Our algorithm (Algorithm~\ref{algo:ptas}) works by partitioning the domain into three parts $T_-, T, T_+$, where 
\begin{align*}
T &= \{\bu \in \mathbb{R}^d: |\iprod{\bw, \bu}| \le  \gamma \sqrt{\opt}\}\\
T_{+} &= \{\bu \in \mathbb{R}^d: \iprod{\bw, \bu} >  \gamma\sqrt{\opt}\}\\
T_{-} &= \{\bu \in \mathbb{R}^d: \iprod{\bw, \bu} < - \gamma\sqrt{\opt}\} \;.
\end{align*} 
The hypothesis $h(\bx)$ behaves as a different function in each of these parts. 
For $\bx \in T_-$, the hypothesis is the $0$ function. For $\bx \in T_+$, the hypothesis 
takes the value of $\iprod{ \bw_{+}, \bx}$, which is the best fitting linear function over $T_+$. 
Finally, over $T$ the hypothesis outputs the value that the best fitting $\ell_1$-norm bounded polynomial 
of degree $1/\eta^3$. Our main theorem of this section is the following: 

\begin{theorem}\label{thm:ptas}
Let $\D_{\bx}$ be $\nu$-subgaussian for $\nu \leq O(1)$, $\|\bw^*\|_2 \le 1$ and $y \in [0, 1]$ for every $(\bx, y) \sim \D$, 
then there is an algorithm that takes $O \left(\frac{1}{\eps^2} \cdot \left( \frac{d}{\eta^3 \nu^2} \right)^{1/\eta^3} \right)$ samples and time, and returns a hypothesis $h$ that with high probability satisfies 
	\[ \E_{\D} \left[ (h(\bx) - y))^2 \right] \leq (1+\eta) \opt + \eps \;. \]
\end{theorem}

\begin{remark}
We note that if the distribution is uniform over $\mathbb{S}^{n-1}$, then the sample complexity of our algorithm scales as $2^{1/\eta^3}$, instead of $d^{1/\eta^3}$, since the distribution is $(1/\sqrt{d})$-subgaussian. 
That is, under the uniform distribution over the unit sphere, the sample complexity is independent of $d$. 
\end{remark}

The proof Theorem~\ref{thm:ptas} follows from a direct application of the following properties of 
Algorithm \ref{algo:ptas} with the specified parameters.

\begin{lemma}\label{lem:three-pieces}
Let $\D_{\bx}$ be $\nu$-subgaussian for $\nu \leq O(1)$ and let $S$ be a set of i.i.d. samples drawn from $\D$. 
If $m = |S|  = \Omega (\frac{k^k \cdot d^k}{\nu^{2k}} \cdot \frac{1}{\epsilon^2})$, 
where $k = \frac{1}{\eta^3}$, $\|\bw^*\|_2 \le 1$ and $y \in [0, 1]$, 
then for $\gamma = \Omega\left( \sqrt{ \log\left(\frac{1}{\eta}\right)} \right)$, we have
\begin{enumerate}
\item $\E_{\D}\left[ \left( \iprod{ \bw_{+}, \bx }  - y\right)^2 1_{T_+}(\bx)  \right] \leq  \E_{\D}\left[ \left( \relu(\iprod{ \bw^*, \bx } ) - y\right)^2 1_{T_+}(\bx)  \right] + \frac{\eta}{3} \cdot \opt + \epsilon\Pr_{\D}[T_+]$.

\item $\E_{\D} \left[ \left( 0 - y\right)^2 1_{T_-}(\bx) \right] \leq \E_{\D}\left[ \left( \relu(\iprod{ \bw^*, \bx } ) - y\right)^2  1_{T_-}(\bx)\right] + \frac{\eta}{3} \cdot \opt$.

\item $ \E_{\D} \left[ \left( P(\bx) - y \right)^2 1_T(\bx) \right] \leq  \E_{\D} \left[ \left( \relu(\iprod{ \bw^*, \bx } ) - y \right)^2 1_T(\bx) \right] + \frac{\eta}{3}\cdot \opt + \epsilon \Pr_{\D}[T]$.
\end{enumerate}
\end{lemma}

\begin{proof}[Proof of Theorem~\ref{thm:ptas}]
Using Lemma~\ref{lem:three-pieces}, we get
\begin{align*}
\E_{\D} \left[ (h(\bx) - y)^2 \right] 
	&= \E_{\D} \left[ (h(\bx) - \relu(\iprod{\bx, \bw^*}))^2 (1_{T_+}(\bx) + 1_T(\bx) + 1_{T_-}(\bx))\right] \\
	&\leq \E_{\D}\left[ \left( \iprod{ \bw_{+}, \bx }  - y\right)^2 1_{T_+}(\bx)  \right] + \E_{\D} \left[ \left( 0 - y\right)^2 1_{T_-}(\bx) \right] + \E_{\D} \left[ \left( P(\bx) - y \right)^2 1_T(\bx) \right] \\
	&\leq \E_{\D}\left[ \left( \relu(\iprod{ \bw^*, \bx } ) - y\right)^2 \right] + \eta \cdot \opt + \epsilon = (1 + \eta) \opt + \eps \;.
\end{align*}
\end{proof}

\begin{algorithm}
	\caption{PTAS for ReLU regression}\label{algo:ptas}
	\hspace*{\algorithmicindent} \textbf{Input} $0 < \opt \le 1$ and access to i.i.d. samples from $\D$\\
	\hspace*{\algorithmicindent} \textbf{Parameters} $r \in \mathbb{N}$, $\epsilon, \gamma, \epsilon, W > 0$
	\begin{algorithmic}[1]
		\State Find $\bw$ using algorithm from previous section. 
		This takes $O(\frac{d}{\eps^2} \log \frac{d}{\delta})$ samples and 
		satisfies $\| \bw - \bw^* \|_2\leq O(1/\nu) \cdot \sqrt{\opt}$. 
		
		\State Let $T = \{\bu \in \mathbb{R}^d: |\iprod{\bw, \bu}| \le  \gamma \sqrt{\opt}\}$, $T_{+} = \{\bu \in \mathbb{R}^d: \iprod{\bw, \bu} >  \gamma\sqrt{\opt}\}$ and $T_{-} = \{\bu \in \mathbb{R}^d: \iprod{\bw, \bu} < - \gamma\sqrt{\opt}\}$.
		\State Find a degree $k = \frac{1}{\eta^3}$, $d$-variate polynomial, $P$, the $\ell_1$-norm of whose 
		coefficients is at most $\nu \cdot O(4^k)= O(4^k)$, using $L_2$-polynomial regression on 
		$m_{pol} = O(\frac{k^k \cdot d^k}{\nu^{2k}} \cdot \frac{1}{\epsilon^2})$ samples such that
		\[
		\E_{S|_T}[(P(\bx) - y)^2] \le \min_{P'\in POL_{r,d}}\E_{\D|_T}[(P'(\bx) - y)^2] + \epsilon \;,
		\]
		where $h_{\bw^*}$ is the optimal ReLU classifier w.r.t. $\D$.
		\State Find $\bw_+ \in B(d, 1)$ using least squares with $m_{ls} = O(1/\epsilon^2)$ to get
		\[
		\E_{S|_{T_+}}[(\iprod{\bw_+, \bx} - y)^2] \le \min_{\bw' \in B(d,W)} \E_{\D|_{T_+}}[(\iprod{\bw', \bx} - y)^2] + \epsilon \;.  
		\]
		\State Output the following classifier:
		\begin{align*}
		h(\bx) = 
		\begin{cases}
		\iprod{\bw_+, \bx}, & \bx \in T_{+}\\
		P(\bx), & \bx \in T\\
		0, & \bx \in T_{-}
		\end{cases}
		\end{align*}
	\end{algorithmic}
\end{algorithm}

\begin{figure}
	\includegraphics[scale=0.75]{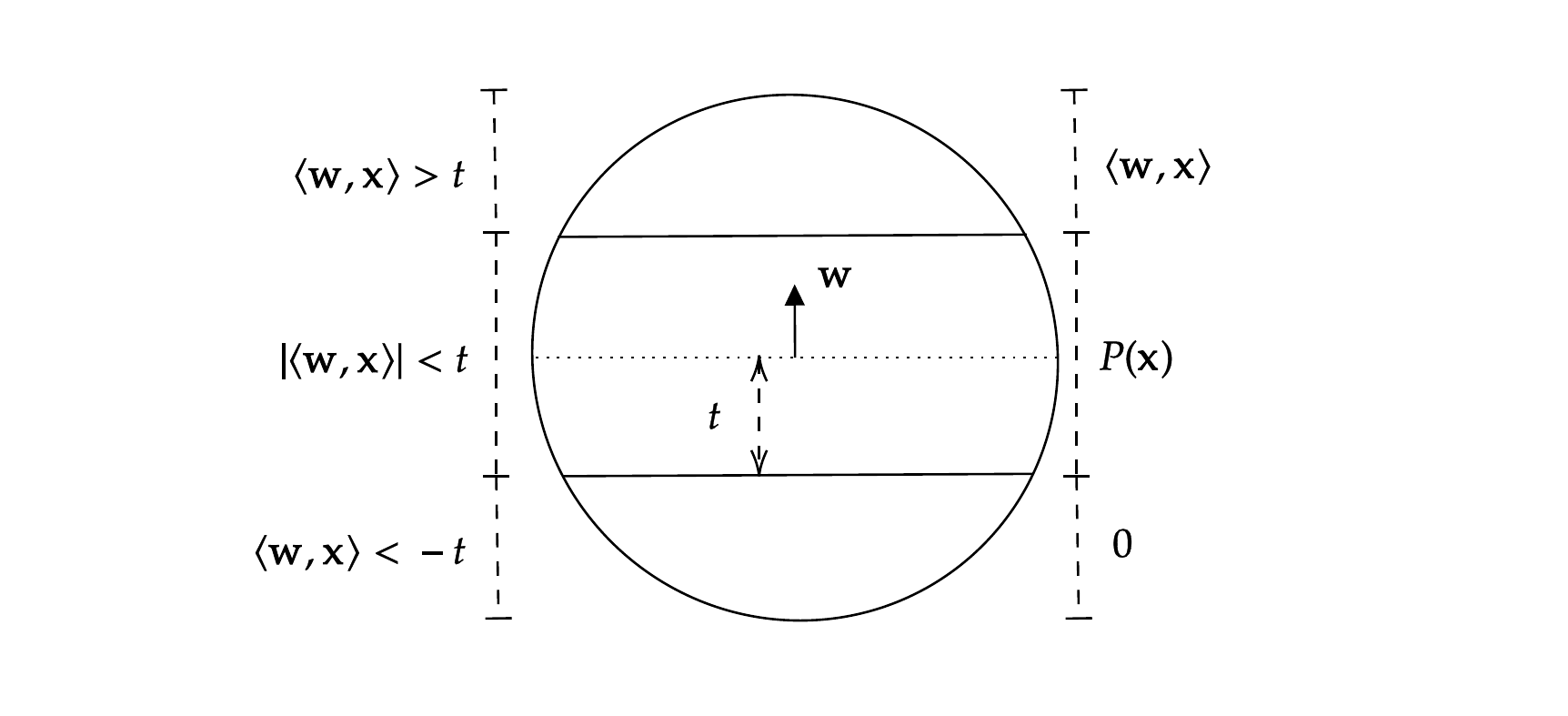}
	\centering
	\caption{We partition the space into three regions depending on $\iprod{\bw, \bx}$. Our hypothesis returns $0$, $\iprod{\bw,\bx}$ or the value of $P(\bx)$ depending on the region. We set $t = \gamma \sqrt{\opt}$.} \label{fig:ptas_informal}
\end{figure}

We now prove Lemma~\ref{lem:three-pieces}. 
\begin{proof}[Proof of Lemma \ref{lem:three-pieces}]
Let $\bw, \bw^*$ be as defined in Algorithm~\ref{algo:ptas}. 
We first project $\D_{\bx}$ down to two dimensions. Let $S = \{ \bx \mid \iprod{\bw^*, \bx} > 0 \}$. Let $V$ be the 2-dimensional space spanned by $\bw, \bw^*$ and let $P_V$ be the orthogonal projection onto $V$. If $\bx \in \overline{S} \land T_+$, then $\iprod{\bw, \bx} \geq \gamma \sqrt{\opt}$ and $\iprod{ \bw^*, \bx }  \leq 0$. Since $\bw$ is a constant factor approximation for a $\nu$-subgaussian distribution, with probability $1-\delta$ we have 
$\| \bw - \bw^*\|_2\leq \frac{c_{\chi}}{\nu} \cdot \sqrt{\opt}$, for some constant $c_{\chi}$. 
This is easy to check via the Cauchy-Schwartz inequality 
and using the structural lemmas from previous subsections. Additionally
	\begin{align}\label{eqn:w^*.x_bounds}
	-\iprod{ \bw^*, \bx }  &= -\iprod{ \bw^*,  P_V(\bx) }
	= - \iprod{ \bw^* - \bw, P_V(\bx)} - \iprod{ \bw, P_V(\bx)} \\
	&\leq - \iprod{ \bw^* - \bw, P_V(\bx)}
	\leq \| \bw^* - \bw \|_2\|P_V(\bx)\|_2\leq \frac{c_{\chi}}{\nu} \sqrt{\opt}\|P_V(\bx)\|_2 \;.
	\end{align}
A similar calculation for $\bx \in S \land T_-$ implies $\iprod{ \bw^*, \bx }  \le \frac{c_{\chi}}{\nu} \sqrt{\opt} \|P_V(\bx)\|_2$. 
We now bound above the error in the region $T_+$. 
Since $\iprod{\bw_{+}, \bx}$ is the best fitting linear function over $T_+$, 
the loss of $\iprod{\bw^*, \bx}$ is necessarily larger than that of $\iprod{\bw_{+}, \bx}$. 
An application of Lemma~\ref{lem:concentration2} in the first step implies
\begin{align*}
	\E_{S}\left[ \left( \iprod{ \bw_{+}, \bx }  - y\right)^2 1_{T_+}(\bx) \right] &=  \min_{\bw'\in B(d, W)}\E_{\D}[(\iprod{\bw',\bx} - y)^2 1_{T_+}(\bx)] + \epsilon \Pr_{\D}[T_+] \\
	&\leq \E_{\D}[(\iprod{ \bw^*, \bx }  - y)^2 1_{T_+}(\bx)] + \epsilon\Pr_{\D}[T_+] \;.
	\end{align*}
Observe that for $\bx \in S$, $\iprod{\bw^*, \bx} = \relu(\iprod{\bw^*, \bx})$. 
Since $1_{T_+}(\bx) = 1_{T_+ \land S}(\bx) + 1_{T_+ \land \overline{S}}(\bx)$, we get
\begin{align*}
\E_{\D}[(\iprod{ \bw^*, \bx }  - y)^2 1_{T_+}(\bx)]
&= \E_{\D}[(\iprod{ \bw^*, \bx }  - y)^2 1_{T_+ \land S}(\bx)] + \E_{\D}[(\iprod{ \bw^*, \bx }  - y)^2 1_{T_+ \land \overline{S}}(\bx)] \\ 
&= \E_{\D}[(\relu(\iprod{ \bw^*, \bx } ) - y)^2 1_{T_+ \land S}(\bx)] + \E_{\D}[y^2 1_{T_+ \land \overline{S}}(\bx)]	\\
&+ \E_{\D}[(\iprod{ \bw^*, \bx } )(\iprod{ \bw^*, \bx }  - 2y) 1_{T_+ \land \overline{S}}(\bx)] \;.
\end{align*}
It remains to show that the terms corresponding to $1_{T_+ \land S}$ contribute a small error overall. 
Note that $\E_{\D}[y^2 1_{T_+ \land \overline{S}}(\bx)] = \E_\D[(y-\relu(\iprod{\bw^*, \bx}))^2 1_{T_+ \land \overline{S}}(\bx)]$. 
This implies
\begin{align*}
\E_{\D}[(\iprod{ \bw^*, \bx }  - y)^2 1_{T_+}(\bx)]	&=  \E_{\D}[(\relu(\iprod{ \bw^*, \bx } ) - y)^21_{T_+}(\bx)] + \E_{\D}[(\iprod{ \bw^*, \bx } )(\iprod{ \bw^*, \bx }  - 2y) 1_{T_+ \land \overline{S}}(\bx)] \\
& \leq \E_{\D}[(\relu(\iprod{ \bw^*, \bx } ) - y)^21_{T_+}(\bx)] + \E_{\D}[|\iprod{ \bw^*, \bx } |(|\iprod{ \bw^*, \bx } | + 2) 1_{T_+ \land \overline{S}}(\bx)] \;.
\end{align*}
To bound above the second term, we use bounds from Equation~\eqref{eqn:w^*.x_bounds}, 
and an application of Lemma~\ref{lem:expectation_bound}.
\begin{align*}
\E_{\D}[|\iprod{ \bw^*, \bx } |(|\iprod{ \bw^*, \bx } | + 2) 1_{T_+ \land \overline{S}}(\bx)] & \leq \E_{\D}\left[\frac{c_{\chi}}{\nu}\sqrt{\opt}\|P_V(\bx)\|_2\left(\frac{c_{\chi}}{\nu}\sqrt{\opt}\|P_V(\bx)\|_2+ 2 \right) 1_{T_+ \land \overline{S}}(\bx)\right]\\
&= \frac{c_{\chi}}{\nu} \cdot \sqrt{\opt} \cdot  \E_{\D}\left[\left(\frac{c_{\chi}}{\nu}\sqrt{\opt}\|P_V(\bx)\|_2^2 + 2\|P_V(\bx)\|_2\right) 1_{T_+ \land \overline{S}}(\bx) \right] \\
&\leq \frac{\eta}{6} \cdot \opt \;.
\end{align*}
Overall, the three equation blocks above imply that the first condition of the lemma is true. 
The analysis of the error of $h$ in $T_-$ is done similarly. We can write:
\begin{align*}
&\E_{\D}\left[ y^2 1_{T_-}(\bx) \right] \\
&= \E_{\D}[y^2 1_{T_- \land S}(\bx)] + \E_{\D}[y^2 1_{T_- \land \overline{S}}(\bx)] \\
&= \E_{\D}[(y - \relu(\iprod{ \bw^*, \bx } ) + \relu(\iprod{ \bw^*, \bx } ))^2 1_{T_- \land S}(\bx)] + \E_{\D}[(y - \relu(\iprod{ \bw^*, \bx } ))^2 1_{T_- \land \overline{S}}(\bx)]\\
&=  \E_{\D}[(y - \relu(\iprod{ \bw^*, \bx } ))^21_{T_-}(\bx)] + \E_{\D}[(\iprod{ \bw^*, \bx } )(2y - (\iprod{ \bw^*, \bx } )) 1_{T_- \land S}(\bx)]\\
&\le  \E_{\D}[(y - \relu(\iprod{ \bw^*, \bx } ))^21_{T_-}(\bx)] + 2\E_{\D}[|\iprod{ \bw^*, \bx } |(2 - |\iprod{ \bw^*, \bx } |) 1_{T_- \land S}(\bx)]\\
&\le \E_{\D}[(y -\relu(\iprod{ \bw^*, \bx } ))^21_{T_-}(\bx)] + \frac{\eta}{6} \cdot \opt \;.
\end{align*}
Finally, we analyze the error of our hypothesis in the region $T$. 
\begin{align*}
&\E_{\D} \left[ \left( P(\bx) - y \right)^2 1_T(\bx) \right]\\
&=  \E_{\D} \left[(\relu(\iprod{\bw^*, \bx}) - y)^2 1_T(\bx) \right] + 2\E_{\D} \left[\left( P(\bx) - \relu(\iprod{\bw^*, \bx})\right) (\relu(\iprod{\bw^*, \bx}) - y ) 1_T(\bx) \right]\\
& + \E_\D\left[ (\relu(\iprod{\bw^*, \bx}) - P(\bx))^2 1_T(\bx)\right] \;.
\end{align*}
The final error term is bounded above via applications of Lemma~\ref{lem:concentration2} 
and Lemma~\ref{lem:poly_approx} by $\frac{\eta^2}{10} \cdot \opt$. 
To bound from above the cross term, we use the fact that $1_T(\bx)^2 = 1_T(\bx)$
\begin{align*}
&2\E_{\D} \left[\left( P(\bx) - \relu(\iprod{\bw^*, \bx})\right) 1_T(\bx) (\relu(\iprod{\bw^*, \bx}) - y) 1_T(\bx) \right] \\
& \leq \sqrt{ \E_{\D} \left[ \left( P(\bx) - \relu(\iprod{\bw^*, \bx})\right)^2 1_T(\bx) \right]}  
\sqrt{\E_{\D} \left[  (\relu(\iprod{\bw^*, \bx}) - y)^2 1_T(\bx) \right] } \\
& \leq \frac{\eta}{10} \cdot \opt \;.
\end{align*}
Putting these together gives us the desired result.	
\end{proof}


\section{Conclusions} \label{sec:conc}
In this work, we gave the first constant approximation scheme for ReLU regression under the assumption of log-concavity. We proved that optimizing a convex surrogate loss suffices for obtaining approximate guarantees. We further proposed a PTAS for ReLU regression under the assumption of sub-gaussianity, which refines the so obtained solution using ideas from localization and polynomial approximation. 

Our work here was focussed on the ReLU activation and we leave open the extensions to other activation functions. We believe that the Chow learnability condition is potentially satisfied under log-concavity for activations that approximate thresholds such as sigmoid.

The underlying surrogate loss approach seems powerful and exploring further applications is an interesting direction for future work. Further, designing approximation schemes for a linear combination of activations functions is an interesting open question.

\subsection*{Acknowledgements}
This work was done in part while the authors were visiting the Simons Institute for the Theory of Computing for the Summer 2019 program on the Foundations of Deep Learning. 
ID was supported by NSF Award CCF-1652862 (CAREER), a Sloan Research Fellowship, and 
a DARPA Learning with Less Labels (LwLL) grant. SG was supported by the JP Morgan AI Phd Fellowship. SK was supported by NSF award CNS 1414082 and ID's startup grant. 
AK was supported by NSF awards CCF 1909204 and CCF 1717896. 
MS was supported by the Packard Fellowship in Science and Engineering, a Sloan Research Fellowship in Mathematics, an NSF-CAREER under award \#1846369, the Air Force Office of Scientific Research Young Investigator Program (AFOSR-YIP) under award \#FA 9550-18-1-0078, DARPA Learning with Less Labels (LwLL) and FastNICs programs, an NSF-CIF award \#1813877, and a Google faculty research award. 

\bibliographystyle{apalike}
\bibliography{allrefs,Bibfiles}

\newpage

\appendix


\section{Useful Properties}
We use the following fact about sub-gaussian distributions.
\begin{fact}\label{lem:concentration}
	If $\D$ is $\nu$-subgaussian then if $P$ is an $d$-variate degree $k$ polynomial then taking an expectation over $m_0$ samples yeilds
	\[ 
	\Pr \left[ | \E_{S}[P(\bx)] - \E_{\D}[P(\bx)] | \geq \eps \right] \leq \exp \left(-\frac{m \eps^2}{\nu^2 {\text Var}[P(\bx)]} \right)^{1/k} \;.
	\]
\end{fact}

\section{Proof of Theorem \ref{thm:mahdi}}\label{sec:mahdi}
We begin by stating a few auxiliary lemmas that play a crucial role in our proof.

\begin{lemma}\label{concen} 
Consider the assumptions of Theorem \ref{thm:mahdi}. 
Also assume $(\vct{x}_i,y_i)_{i=1}^m$ are generated i.i.d.~with $\vct{x}_i$ having a log-concave 
marginal and $y_i$ obeying $\abs{y_i}\le 1$. Furthermore, assume $\twonorm{\widehat{\vct{w}}}\le W$ 
and $\sigma:\R\rightarrow\R$ is an activation obeying $\abs{\sigma(z)}\le B\abs{z}$. Then, as long as
\begin{align*}
m \geq \tilde{\Omega} \left( \frac{d}{\xi^2} \log^4 (d/\delta)\left(W+1\right)^2 \right) \;,
\end{align*}
we have that
\begin{align*}
\twonorm{\frac{1}{m}\sum_{i=1}^m\left(\sigma\left(\langle \widehat{\vct{w}},\vct{x}_i\rangle\right)-y_i\right)\vct{x}_i}\le \xi
\end{align*}
holds with probability at least $1-\delta$.
\end{lemma}
Next we show that the gradient of the surrogate loss obeys a certain correlation inequality with the proof deferred to end of the section.
\begin{lemma}
\label{corineq}
As long as $\epsilon \le W$, $\delta \ge e^{-O(\sqrt{d})}$ and
\begin{align*}
m \geq \tilde{\Omega} \left(  \frac{\gamma^2}{\mu^2}d \log\left(\frac{W + 1}{\mu\delta}\right) \right) \;,
\end{align*}
we have
\begin{align}
\label{cond}
\langle \nabla \widehat{L}^\sur(\vct{w})-\nabla \widehat{L}^\sur(\widehat{\vct{w}}),\vct{w}-\widehat{\vct{w}}\rangle\ge \alpha \twonorm{\vct{w}-\widehat{\vct{w}}}^2+\beta\twonorm{\nabla \widehat{L}^\sur(\vct{w})-\nabla \widehat{L}^\sur(\widehat{\vct{w}})}^2
\end{align}
holds for all $\vct{w}\in\R^d$ obeying $\frac{\epsilon}{3}\le\twonorm{\vct{w}-\widehat{\vct{w}}}\le 2W$ with $\alpha=\frac{\mu}{3}$ and $\beta=\frac{1}{8}$ with probability at least $1- \delta$.
\end{lemma}
With these two key lemmas in place we are now ready to prove the main theorem. First note that since ${\vct{w}_0} = 0$ we have $\twonorm{\vct{w}_0-\widehat{\vct{w}}}\le W \le 2W$ and thus by Lemma \ref{corineq} the correlation inequality \eqref{cond} holds at $\vct{w}_0$ with high probability. Furthermore, as we show next when the correlation inequality \eqref{cond} holds subsequent iterations also obey $\twonorm{\vct{w}_\tau-\widehat{\vct{w}}}\le 4W$ allowing us to apply the correlation inequality \eqref{cond} in an inductive fashion.

Let us now consider the progress from one iteration to the next. We can write:
\begin{align*}
\twonorm{\vct{w}_{\tau+1}-\widehat{\vct{w}}}^2=&\twonorm{\vct{w}_{\tau}-\widehat{\vct{w}}}^2-\eta\langle \nabla \widehat{L}^\sur(\vct{w}_\tau),\vct{w}_\tau-\widehat{\vct{w}}\rangle+\eta^2 \twonorm{\nabla \widehat{L}^\sur(\vct{w}_\tau)}^2\\
=&\twonorm{\vct{w}_{\tau}-\widehat{\vct{w}}}^2-\eta\langle \nabla \widehat{L}^\sur(\vct{w}_\tau)-\nabla \widehat{L}^\sur(\widehat{\vct{w}}),\vct{w}_\tau-\widehat{\vct{w}}\rangle-\eta\langle \nabla \widehat{L}^\sur(\widehat{\vct{w}}),\vct{w}_\tau-\widehat{\vct{w}}\rangle+\eta^2 \twonorm{\nabla \widehat{L}^\sur(\vct{w}_\tau)}^2\\
\le&\twonorm{\vct{w}_{\tau}-\widehat{\vct{w}}}^2-\eta\langle \nabla \widehat{L}^\sur(\vct{w}_\tau)-\nabla \widehat{L}^\sur(\widehat{\vct{w}}),\vct{w}_\tau-\widehat{\vct{w}}\rangle-\eta\langle \nabla \widehat{L}^\sur(\widehat{\vct{w}}),\vct{w}_\tau-\widehat{\vct{w}}\rangle\\
&+2\eta^2 \twonorm{\nabla \widehat{L}^\sur(\vct{w}_\tau)-\nabla \widehat{L}^\sur(\widehat{\vct{w}})}^2+2\eta^2\twonorm{\nabla \widehat{L}^\sur(\widehat{\vct{w}})}^2\\
\overset{(a)}{\le}&\left(1-\eta\alpha\right)\twonorm{\vct{w}_{\tau}-\widehat{\vct{w}}}^2-\eta(\beta-2\eta)\twonorm{\nabla \widehat{L}^\sur(\vct{w}_\tau)-\nabla \widehat{L}^\sur(\widehat{\vct{w}})}^2\\
&+2\eta^2\twonorm{\nabla \widehat{L}^\sur(\widehat{\vct{w}})}^2-\eta\langle \nabla \widehat{L}^\sur(\widehat{\vct{w}}),\vct{w}_\tau-\widehat{\vct{w}}\rangle\\
\overset{(b)}{\le}&\left(1-\frac{\alpha\eta}{2}\right)\twonorm{\vct{w}_{\tau}-\widehat{\vct{w}}}^2-\eta(\beta-2\eta)\twonorm{\nabla \widehat{L}^\sur(\vct{w}_\tau)-\nabla \widehat{L}^\sur(\widehat{\vct{w}})}^2\\
&+\eta\left(2\eta+\frac{1}{2\alpha}\right)\twonorm{\nabla \widehat{L}^\sur(\widehat{\vct{w}})}^2\\
\overset{(c)}{\le}&\left(1-\frac{\alpha\eta}{2}\right)\twonorm{\vct{w}_{\tau}-\widehat{\vct{w}}}^2+\eta\left(2\eta+\frac{1}{2\alpha}\right)\twonorm{\nabla \widehat{L}^\sur(\widehat{\vct{w}})}^2\\
\overset{(d)}{\le}&\left(1-\frac{\alpha\eta}{2}\right)\twonorm{\vct{w}_{\tau}-\widehat{\vct{w}}}^2+\eta\left(2\eta+\frac{1}{2\alpha}\right)\frac{\alpha^2}{24\alpha\beta+12}\epsilon^2 \;.
\end{align*}
Here, (a) follows from \eqref{cond} and (b) from $\langle \vct{a} ,\vct{b}\rangle\le \frac{1}{2\alpha}\twonorm{\vct{a}}^2+\frac{\alpha}{2}\twonorm{\vct{b}}^2$, (c) from $\eta\le \frac{\beta}{2}$, and (d) from Lemma \ref{concen} with $\xi=\frac{\alpha\epsilon}{\sqrt{18\alpha\beta+9}}=\frac{\mu\epsilon}{3\sqrt{\frac{3\mu}{4}+1}}$. Thus, iterating the above in all subsequent iterations we have
\begin{align*}
\twonorm{\vct{w}_\tau-\widehat{\vct{w}}}^2\le& W^2+\frac{\eta\left(2\eta+\frac{1}{2\alpha}\right)}{1-\left(1-\frac{\alpha\eta}{2}\right)}\frac{\alpha^2}{2\alpha\beta+1}\epsilon^2\\
=&W^2+\frac{1}{\alpha}\left(4\eta+\frac{1}{\alpha}\right)\frac{\alpha^2}{18\alpha\beta+9}\epsilon^2\\
\le&W^2+\frac{1}{\alpha}\left(2\beta+\frac{1}{\alpha}\right)\frac{\alpha^2}{18\alpha\beta+9}\epsilon^2\\
\le&W^2+ \frac{\epsilon^2}{9}\\
<& 4W^2 \;,
\end{align*}
where in the last inequality we used the fact that $\epsilon\le W$.Therefore, $\twonorm{\vct{w}_\tau-\widehat{\vct{w}}}\le 2W$ for all $\tau\ge 1$ and use of the correlation inequality is justified. Furthermore, iterating the above lemma we conclude that as long as $\twonorm{\vct{w}_\tau - \widehat{\vct{w}}}\ge \frac{\epsilon}{3}$ it holds
\begin{align}
\label{tmpwrd}
\twonorm{\vct{w}_{\tau}-\widehat{\vct{w}}}^2\le& \left(1-\frac{\alpha \eta}{2}\right)^\tau\twonorm{\vct{w}_0-\widehat{\vct{w}}}^2+\frac{\eta\left(2\eta+\frac{1}{2\alpha}\right)\xi^2}{1-\left(1-\frac{\alpha \eta}{2}\right)}\nonumber\\
=&\left(1-\frac{\alpha \eta}{2}\right)^\tau\twonorm{\vct{w}_0-\widehat{\vct{w}}}^2+\frac{1}{\alpha}\left(4\eta+\frac{1}{\alpha}\right)\xi^2\nonumber\\
\le&\left(1-\frac{\alpha \eta}{2}\right)^\tau\twonorm{\vct{w}_0-\widehat{\vct{w}}}^2+\frac{1}{\alpha}\left(2\beta+\frac{1}{\alpha}\right)\xi^2\nonumber\\
=&\left(1-\frac{\alpha \eta}{2}\right)^\tau\twonorm{\vct{w}_0-\widehat{\vct{w}}}^2+\frac{\epsilon^2}{9} \;.
\end{align}
Thus, after $\tau\ge T:=\frac{2\log\left(\frac{\epsilon}{3W}\right)}{\log\left(1-\frac{\mu\eta}{6}\right)}$ we have 
\begin{align*}
\twonorm{\vct{w}_{\tau}-\widehat{\vct{w}}}^2\le \frac{\epsilon^2}{9}+\frac{\epsilon^2}{9}=\frac{2\epsilon^2}{9}\quad \Rightarrow \quad \twonorm{\vct{w}_{\tau}-\widehat{\vct{w}}}\le \frac{2}{3}\epsilon \;.
\end{align*}
Note that above was carried out under the assumption that for all $t=1,2,\ldots,T$ we have $\twonorm{\vct{w}_{t}-\widehat{\vct{w}}}\ge \frac{\epsilon}{3}$. We note that if this assumption is violated at some iteration $
\widetilde{t}$ we have $\twonorm{\vct{w}_{\widetilde{t}}-\widehat{\vct{w}}}\le \frac{\epsilon}{3}$. Now either $\twonorm{\vct{w}_{\tau}-\widehat{\vct{w}}}\le \frac{\epsilon}{3}$ for all $\tau\ge \widetilde{t}$ in which case after $\tau\ge T:=\frac{2\log\left(\frac{\epsilon}{3W}\right)}{\log\left(1-\frac{\mu\eta}{6}\right)}$ we have $\twonorm{\vct{w}_{\tau}-\widehat{\vct{w}}}\le \frac{1}{3}\epsilon$. If not at some iteration $t\ge\widetilde{t}$ we have $\twonorm{\vct{w}_{t}-\widehat{\vct{w}}}\le \frac{1}{3}\epsilon$ and $\twonorm{\vct{w}_{t+1}-\widehat{\vct{w}}}\ge  \frac{1}{3}\epsilon$. Thus,
\begin{align*}
\twonorm{\vct{w}_{t+1}-\widehat{\vct{w}}}=&\twonorm{\vct{w}_{t}-\widehat{\vct{w}}-\eta\nabla \widehat{L}^\sur(\widehat{\vct{w}})(\vct{w}_{t})}\\
=&\twonorm{\vct{w}_{t}-\widehat{\vct{w}}-\eta\left(\nabla \widehat{L}^\sur(\widehat{\vct{w}})(\vct{w}_{t
})-\nabla \widehat{L}^\sur(\widehat{\vct{w}})(\widehat{\vct{w}})\right)-\eta\nabla \widehat{L}^\sur(\widehat{\vct{w}})(\widehat{\vct{w}})}\\
\le&\twonorm{\vct{w}_{t}-\widehat{\vct{w}}-\eta\left(\nabla \widehat{L}^\sur(\widehat{\vct{w}})(\vct{w}_{t
})-\nabla \widehat{L}^\sur(\widehat{\vct{w}})(\widehat{\vct{w}})\right)}+\eta\twonorm{\nabla \widehat{L}^\sur(\widehat{\vct{w}})(\widehat{\vct{w}})}\\
\overset{(a)}{\le}&\twonorm{\vct{w}_{t}-\widehat{\vct{w}}}+\eta\twonorm{\nabla \widehat{L}^\sur(\widehat{\vct{w}})(\widehat{\vct{w}})}\\
\overset{(b)}{\le}&\frac{2}{3}\epsilon \le 2W \;.
\end{align*}
In the above, (a) follows from the fact that $\widetilde{\sigma}$ is convex which implies that for any two scalars $z, \widehat{z}$ we have $\left( \widetilde{\sigma}'(z)- \widetilde{\sigma}'(\widehat{z})\right)(z-\widehat{z})\ge 0$ which implies $\left( \sigma(z)- \sigma(\widehat{z})\right)(z-\widehat{z})\ge 0$. This in turn implies that 
 \begin{align*}
\langle \nabla \widehat{L}^\sur(\vct{w})-\nabla \widehat{L}^\sur(\widehat{\vct{w}}),\vct{w}-\widehat{\vct{w}}\rangle=& \frac{1}{m}\sum_{i=1}^m\left(\sigma\left(\langle \vct{w},\vct{x}_i\rangle\right)-\sigma\left(\langle \widehat{\vct{w}},\vct{x}_i\rangle\right)\right)\left(\vct{x}_i^T(\vct{w}-\widehat{\vct{w}})\right)\ge 0 \;,
\end{align*} 
so that 
$\twonorm{\vct{w}_{t}-\widehat{\vct{w}}-\eta\left(\nabla \widehat{L}^\sur(\widehat{\vct{w}})(\vct{w}_{t})-\nabla \widehat{L}^\sur(\widehat{\vct{w}})(\widehat{\vct{w}})\right)}\le \twonorm{\vct{w}_t-\widehat{\vct{w}}}$. 
Also (b) follows from the fact that $\eta\le \frac{\beta}{2}=\frac{1}{16}$ and Lemma \ref{concen} with $\xi\le\frac{16}{3}\epsilon$. 
As a result, we are in a region where the correlation inequality applies. 
Furthermore, using an argument similar to \eqref{tmpwrd} for all $\tau\ge t$, where 
$\twonorm{\vct{w}_\tau-\widehat{\vct{w}}}\ge \frac{\epsilon}{3}$, we have
\begin{align}
\label{tmpwrd}
\twonorm{\vct{w}_{\tau}-\widehat{\vct{w}}}^2\le\left(1-\frac{\alpha \eta}{2}\right)^{\tau-(t+1)}\twonorm{\vct{w}_{t+1}-\widehat{\vct{w}}}^2+\frac{\epsilon^2}{9}\le \twonorm{\vct{w}_{t+1}-\widehat{\vct{w}}}^2+\frac{\epsilon^2}{9}\le \frac{5}{9}\epsilon^2\quad\Rightarrow \twonorm{\vct{w}_{\tau}-\widehat{\vct{w}}}\le \epsilon \;.
\end{align}
Of course, if at some point we again have $\twonorm{\vct{w}_\tau-\widehat{\vct{w}}}\le \frac{\epsilon}{3}$, 
we repeat the above arguments. In conclusion, in all cases after 
$\tau\ge T:=\frac{2\log\left(\frac{\epsilon}{9W}\right)}{\log\left(1-\frac{\mu\eta}{6}\right)}$, 
we have $\twonorm{\vct{w}_\tau-\widehat{\vct{w}}}\le \epsilon$ completing the proof.

\subsection{Proof of Lemma \ref{concen}}
Define the random vector
\begin{align*}
\vct{z}=\frac{1}{m}\sum_{i=1}^m\left(\sigma\left(\langle \widehat{\vct{w}},\vct{x}_i\rangle\right)-y_i\right)\vct{x}_i.
\end{align*}
Note that
\begin{align*}
\twonorm{\vct{z}}^2=\sum_{j=1}^m \abs{\vct{z}_j}^2,
\end{align*}
so that it suffices to bound the square of the individual entries of the vector $\vct{z}$. 
To bound this quantity we bound individual entires of the vector. Note that any such entry can be written in the form
\begin{align*}
\frac{1}{m}\sum_{i=1}^m\left(\sigma\left(\twonorm{\widehat{\vct{w}}}Z_i\right)-y_i\right)X_i \;,
\end{align*} 
where $Z_i=\langle \frac{\widehat{\vct{w}}}{\twonorm{\widehat{\vct{w}}}},\vct{x}_i \rangle$ 
and $X_i=\langle \vct{x}_i,\vct{e}_j\rangle$ are sub-exponential random variables with constant 
$\|\cdot\|_{\psi_1}$ norm (where $\| \cdot \|_{\psi_p} := \inf \{ k \in (0, \infty) \mid \E[\exp((|x|/k)^p)-1] \leq 1\}$ -- this characterizes the limiting behavior of the probability density function). 
Furthermore, $\abs{y_i}\le 1$ implying that $y_i$ is a sub-Gaussian random variable. Therefore, 
\begin{align*}
\|\sigma\left(\twonorm{\widehat{\vct{w}}}Z_i\right)-y_i\|_{\psi_1}\le c\left(W+1\right)\quad\Rightarrow\quad\|\left(\sigma\left(\twonorm{\widehat{\vct{w}}}Z_i\right)-y_i\right)X_i\|_{\psi_{1/2}}\le c(W+1).
\end{align*}
Thus, using a well-known result of Talagrand 
(specifically combining \cite[Theorem 6.21]{ledoux2013probability} and \cite[Lemma 2.2.2]{shorack2009empirical}), 
we have
\begin{align*}
\|\frac{1}{m}\sum_{i=1}^m\left(\sigma\left(\twonorm{\widehat{\vct{w}}}Z_i\right)-y_i\right)X_i\|_{\psi_{1/2}}\le c\frac{\log m}{\sqrt{m}}\left(W+1\right) \;.
\end{align*}
Therefore,
\begin{align*}
\mathbb{P}\Bigg\{\frac{1}{m}\sum_{i=1}^m\left(\sigma\left(\twonorm{\widehat{\vct{w}}}Z_i\right)-y_i\right)X_i\ge ct\frac{\log m}{\sqrt{m}}\left(W+1\right)\Bigg\}\le C e^{-\sqrt{t}}.
\end{align*}
Thus, using the union bound
\begin{align*}
\mathbb{P}\Big\{\twonorm{\vct{z}}\ge c\frac{\sqrt{d}}{\sqrt{m}} t\log m\left(W+1\right)\Big\}\le& d\mathbb{P}\Bigg\{\frac{1}{m}\sum_{i=1}^m\left(\sigma\left(\twonorm{\widehat{\vct{w}}}Z_i\right)-y_i\right)X_i\ge c t \frac{\log m}{\sqrt{m}}\left(W+1\right)\Bigg\}\\
\le& dCe^{-\sqrt{t}}.
\end{align*}
Setting $t = \log^2(dC/\delta)$ completes the proof.

\subsection{Proof of Lemma \ref{corineq}}
\label{conineqpf}
For any vector $\vct{w}\in\R^d$ and $\widehat{\vct{w}}$, we have
\begin{align*}
\frac{1}{\twonorm{\vct{w}-\widehat{\vct{w}}}^2}\langle \nabla \widehat{L}^\sur(\vct{w})-\nabla \widehat{L}^\sur(\widehat{\vct{w}}),\vct{w}-\widehat{\vct{w}}\rangle=& \frac{1}{m}\sum_{i=1}^m\frac{\left(\sigma\left(\langle \vct{w},\vct{x}_i\rangle\right)-\sigma\left(\langle \widehat{\vct{w}},\vct{x}_i\rangle\right)\right)\left(\vct{x}_i^T(\vct{w}-\widehat{\vct{w}})\right)}{\twonorm{\vct{w}-\widehat{\vct{w}}}^2}\\
:=&\frac{1}{m}\sum_{i=1}^m \mathcal{Y}_i(\vct{w}) \;,
 \end{align*}
where we define the random processes 
$\mathcal{Y}_i(\vct{w}):=\frac{\left(\sigma\left(\langle \vct{w},\vct{x}_i\rangle\right)-\sigma\left(\langle \widehat{\vct{w}},\vct{x}_i\rangle\right)\right)\left(\vct{x}_i^T(\vct{w}-\widehat{\vct{w}})\right)}{\twonorm{\vct{w}-\widehat{\vct{w}}}^2}$.

Thus, for the random process $\mathcal{X}_i(\vct{w}):=\E[\mathcal{Y}_i(\vct{w})]-\mathcal{Y}_i(\vct{w})$ we have
\begin{align*}
\mathcal{X}_i(\vct{w})=&\E[\mathcal{Y}_i(\vct{w})]-\mathcal{Y}_i(\vct{w})\\
\overset{(a)}{\le}& \E[\mathcal{Y}_i(\vct{w})]\\
\overset{(b)}{\le}& \frac{\E\big[\abs{\vct{x}_i^T(\vct{w}-\widehat{\vct{w}})}^2\big]}{\twonorm{\vct{w}-\widehat{\vct{w}}}^2}\\
\overset{(c)}{=}&1.
\end{align*}
Here, (a) follows from the fact that $\widetilde{\sigma}$ is convex which implies that for any two scalars $z, \widehat{z}$ we have
 \begin{align*}
\left( \widetilde{\sigma}'(z)- \widetilde{\sigma}'(\widehat{z})\right)(z-\widehat{z})\ge 0\quad\Rightarrow\quad \left( \sigma(z)- \sigma(\widehat{z})\right)(z-\widehat{z})\ge 0 \;.
 \end{align*}
Thus, we always have $\mathcal{Y}_i(\vct{w}):=\left(\sigma\left(\langle \vct{w},\vct{x}_i\rangle\right)-\sigma\left(\langle \widehat{\vct{w}},\vct{x}_i\rangle\right)\right)\left(\vct{x}_i^T(\vct{w}-\widehat{\vct{w}})\right)\ge 0$, (b) from 1-Lipscitzness, and (c) from the isotropic assumption on $\vct{x}_i$. 
 
 Also we have
\begin{align*}
\E[\mathcal{X}_i^2(\vct{w})]=\E[\mathcal{Y}_i^2(\vct{w})]-\left(\E[\mathcal{Y}_i(\vct{w})]\right)^2\le \E[\mathcal{Y}_i^2(\vct{w})]\le \E\big[\left(\vct{x}_i^T(\vct{w}-\widehat{\vct{w}})\right)^4\big]\le \gamma^2 \twonorm{\vct{w}-\widehat{\vct{w}}}^4 \;,
\end{align*}
where in the penultimate inequality we used $1$-Lipschitz property of ReLU 
and in the last inequality we used boundedness of fourth moments of the distribution. 
We will now apply Lemma 7.13 of \cite{candes2015phase} (also see \citep{bentkus2003inequality}) 
for a fixed $\vct{w}$ with $v=\gamma $, $b=1$, and $y=m\xi$ to conclude that
\begin{align}
\label{probineq}
\Pr \left[ \frac{1}{m}\sum_{i=1}^m \mathcal{X}_i(\vct{w})\ge \xi\mu \right] \le e^{-m\frac{\mu^2}{\gamma^2} \xi^2} \;.
\end{align}
Now note that by $\mu$-strong convexity of the surrogate loss, we have
\begin{align*}
\frac{1}{m}\sum_{i=1}^m \mathcal{Y}_i(\vct{w})=
&\frac{1}{m}\sum_{i=1}^m \E[\mathcal{Y}_i(\vct{w})]-\frac{1}{m}\sum_{i=1}^m \mathcal{X}_i(\vct{w})\\
\ge&~\mu-\frac{1}{m}\sum_{i=1}^m \mathcal{X}_i(\vct{w}) \;.
\end{align*}
Using \eqref{probineq} in the latter, we conclude that
\begin{align*}
\frac{1}{m}\sum_{i=1}^m \mathcal{Y}_i(\vct{w})\ge (1-\xi)\mu \;,
\end{align*}
holds with probability at least $1-e^{-m\frac{\mu^2}{\gamma^2} \xi^2}$. 
To continue, define $\vct{h}=\frac{\vct{w}-\widehat{\vct{w}}}{\twonorm{\vct{w}-\widehat{\vct{w}}}}$
and $s=\twonorm{\vct{w}-\widehat{\vct{w}}}$, and note that $\mathcal{Y}_i$ can be alternatively 
be written in the form of the stochastic process
\begin{align*}
\mathcal{Y}_i(\vct{h};s):=\frac{\left(\sigma\left(\langle \widehat{\vct{w}},\vct{x}_i\rangle+s\langle \vct{h},\vct{x}_i\rangle\right)-\sigma\left(\langle \widehat{\vct{w}},\vct{x}_i\rangle\right)\right)}{s}\left(\vct{x}_i^T\vct{h}\right) \;.
\end{align*}
Thus, based on the argument above for a fixed $\vct{h}\in\mathbb{S}^{d-1}$ and a fixed $0\le s\le CW$, 
we have that
\begin{align}
\label{AC1}
\mathcal{Z}(\vct{h};s):=\frac{1}{m}\sum_{i=1}^m \mathcal{Y}_i(\vct{h};s)\ge \left(1-\frac{\xi}{3}\right)\mu
\end{align}
holds with probability at least $1-e^{-m\frac{\mu^2}{9\gamma^2} \xi^2}$. 
To continue, we prove the following simple lemma.
\begin{lemma} 
For any $\vct{h},\widetilde{\vct{h}}\in\mathbb{S}^{d-1}$, we have
\begin{align*}
\abs{\mathcal{Y}_i(\vct{h};s)-\mathcal{Y}_i(\widetilde{\vct{h}};s)}\le \left(\abs{\vct{x}_i^T\vct{h}}+\abs{\vct{x}_i^T\widetilde{\vct{h}}}+1\right)\abs{\vct{x}_i^T(\vct{h}-\widetilde{\vct{h}})} \;.
\end{align*}
\end{lemma}
\begin{proof}
Define
\begin{align*}
f(z):=\frac{\left(\sigma\left(x+sz\right)-\sigma\left(x\right)\right)}{s}z
\end{align*}
and note that for some $0\le t\le 1$ we have
\begin{align*}
\abs{f(z)-f(\widetilde{z})}=&\abs{f'(tz+(1-t)\widetilde{z})\left(z-\widetilde{z}\right)}\\
=&\abs{\frac{\left(\sigma\left(x+s(tz+(1-t)\widetilde{z})\right)-\sigma\left(x\right)\right)}{s}+\sigma'\left(x+s(tz+(1-t)\widetilde{z})\right)(tz+(1-t)\widetilde{z})} \abs{z-\widetilde{z}}\\
\le&\frac{\abs{\left(\sigma\left(x+s(tz+(1-t)\widetilde{z})\right)-\sigma\left(x\right)\right)}}{s}\abs{z-\widetilde{z}}+\abs{\sigma'\left(x+s(tz+(1-t)\widetilde{z})\right)(tz+(1-t)\widetilde{z})}\abs{z-\widetilde{z}}\\
\le&\abs{tz+(1-t)\widetilde{z}}\abs{z-\widetilde{z}}+L\abs{z-\widetilde{z}}\\
\le&\left(\abs{z}+\abs{\widetilde{z}}\right)\abs{z-\widetilde{z}}+\abs{z-\widetilde{z}} \;.
\end{align*}
The proof is complete by noting that $\mathcal{Y}_i(\vct{h};s)=f(\vct{x}_i^T\vct{h})$.
\end{proof}
Define a matrix $\mtx{X}\in\R^{m\times d}$ with rows given by $\vct{x}_1, \vct{x}_2, \ldots, \vct{x}_n\in\R^d$ and note that using 
the previous lemma allows us to handle the deviation of the process $\mathcal{Z}(\vct{h};s)$. 
Specifically, using the triangle inequality we have
\begin{align*}
\abs{\mathcal{Z}(\vct{h};s)-\mathcal{Z}(\widetilde{\vct{h}};s)}\le& \frac{1}{m}\sum_{i=1}^m \left(\abs{\vct{x}_i^T\vct{h}}+\abs{\vct{x}_i^T\widetilde{\vct{h}}}+1\right)\abs{\vct{x}_i^T(\vct{h}-\widetilde{\vct{h}})}\\
=& \frac{1}{m}\sum_{i=1}^m \left(\abs{\vct{x}_i^T\vct{h}}+\abs{\vct{x}_i^T\widetilde{\vct{h}}}\right)\abs{\vct{x}_i^T(\vct{h}-\widetilde{\vct{h}})}+\frac{1}{m}\sum_{i=1}^m \abs{\vct{x}_i^T(\vct{h}-\widetilde{\vct{h}})}\\
\le& \sqrt{\frac{1}{m}\sum_{i=1}^m \left(\abs{\vct{x}_i^T\vct{h}}+\abs{\vct{x}_i^T\widetilde{\vct{h}}}\right)^2}\sqrt{\frac{1}{m}\sum_{i=1}^m \abs{\vct{x}_i^T(\vct{h}-\widetilde{\vct{h}})}^2}+\sqrt{\frac{1}{m}\sum_{i=1}^m \abs{\vct{x}_i^T(\vct{h}-\widetilde{\vct{h}})}^2}\\
=& \frac{1}{m}\twonorm{\abs{\mtx{X}\vct{h}}+\abs{\mtx{X}\widetilde{\vct{h}}}}\twonorm{\mtx{X}(\vct{h}-\widetilde{\vct{h}})}+\frac{1}{\sqrt{m}}\twonorm{\mtx{X}(\vct{h}-\widetilde{\vct{h}})}\\
\le& \frac{1}{m}\left(\twonorm{\mtx{X}\vct{h}}+\twonorm{\mtx{X}\widetilde{\vct{h}}}\right)\twonorm{\mtx{X}(\vct{h}-\widetilde{\vct{h}})}+\frac{1}{\sqrt{m}}\twonorm{\mtx{X}(\vct{h}-\widetilde{\vct{h}})}\\
\le& \frac{2}{m}\opnorm{\mtx{X}}^2\twonorm{\vct{h}-\widetilde{\vct{h}}}+\frac{1}{\sqrt{m}}\opnorm{\mtx{X}}\twonorm{\vct{h}-\widetilde{\vct{h}}} \;.
\end{align*}
To continue further, note that under the log-concave density, 
centered and isotropy assumption using \citep{adamczak2010quantitative}, 
as long as $m\ge Cd$, we have that
\begin{align*}
\opnorm{\mtx{X}}\le 2\sqrt{m}
\end{align*}
holds with probability at least $1-e^{-c\sqrt{d}}$. Thus, 
for any $\vct{h},\widetilde{\vct{h}}\in\mathbb{S}^{d-1}$ and any $s\le CW$ we have
\begin{align}
\label{AC2}
\abs{\mathcal{Z}(\vct{h};s)-\mathcal{Z}(\widetilde{\vct{h}};s)}\le 10 \twonorm{\vct{h}-\widetilde{\vct{h}}}
\end{align}
holds with high probability. Now let us consider an $\mathcal{N}_{\eta}$-cover of the unit sphere with $\eta:=\frac{\xi\mu}{30}$. 
Using the union bound combined with \eqref{AC1} for any $\widetilde{\vct{h}}\in \mathcal{N}_\eta$, we have 
that $\mathcal{Z}(\widetilde{\vct{\vct{h}}};s)\ge \left(1-\frac{\xi}{3}\right)\mu$ holds with probability at least 
\begin{align*}
1-\left(\frac{3}{\eta}\right)^de^{-m\frac{\mu^2}{9\gamma^2} \xi^2}=1-\left(\frac{90}{\xi\mu}\right)^de^{-m\frac{\mu^2}{9\gamma^2} \xi^2}=1-e^{d\log\left(\frac{90}{\xi\mu}\right)-m\frac{\mu^2}{9\gamma^2} \xi^2}\ge 1- e^{-m\frac{\mu^2}{10\gamma^2} \xi^2} \;,
\end{align*}
as long as $m\ge 90 \frac{\gamma^2}{\mu^2}d \frac{\log\left(\frac{90}{\xi\mu}\right)}{\xi^2}$. 
Therefore, using  \eqref{AC2} for any $\vct{h}\in\mathbb{S}^{d-1}$ there exists $\widetilde{\vct{h}}\in\mathcal{N}_\eta$ with 
$\twonorm{\vct{h}-\widetilde{\vct{h}}}\le \eta:=\frac{\xi}{30}$. Thus,
\begin{align*}
\mathcal{Z}(\vct{h};s)\ge \mathcal{Z}(\widetilde{\vct{h}};s)-\abs{\mathcal{Z}(\vct{h};s)-\mathcal{Z}(\widetilde{\vct{h}};s)}\ge \left(1-\frac{2}{3}\xi\right)\mu \;.
\end{align*}
In conclusion, for all $\vct{h}\in\mathbb{S}^{d-1}$ and a fixed $0\le s\le CW$ we have that
\begin{align}
\label{unifh}
\mathcal{Z}(\vct{h};s)\ge \left(1-\frac{2}{3}\xi\right)\mu
\end{align}
holds with probability at least $1- e^{-m\frac{\mu^2}{10\gamma^2} \xi^2}-e^{-c\sqrt{d}}$. 
We now turn our attention to making the result also hold uniformly for all $\frac{\epsilon}{3}\le s \le CW$. 
To this aim, we state the following lemma.
\begin{lemma}
Let $s \geq \eps/3$, then 
\begin{align*}
\abs{\mathcal{Y}_i(\vct{h};s)-\mathcal{Y}_i(\vct{h};\widetilde{s})}\le \frac{6}{\epsilon}\abs{\vct{x}_i^T\vct{h}}^2\abs{ s-\widetilde{s}} \;.
\end{align*}
\end{lemma}
\begin{proof}
Define
\begin{align*}
f(s):=\frac{\left(\sigma\left(\langle \widehat{\vct{w}},\vct{x}_i\rangle+s\langle \vct{h},\vct{x}_i\rangle\right)-\sigma\left(\langle \widehat{\vct{w}},\vct{x}_i\rangle\right)\right)}{s}\left(\vct{x}_i^T\vct{h}\right)
\end{align*}
and note that for $s\ge \frac{\epsilon}{3}$
\begin{align*}
\abs{f'(s)}=&\abs{\frac{s\langle \vct{h},\vct{x}_i\rangle\sigma'\left(\langle \widehat{\vct{w}},\vct{x}_i\rangle+s\langle \vct{h},\vct{x}_i\rangle\right)-\left(\sigma\left(\langle \widehat{\vct{w}},\vct{x}_i\rangle+s\langle \vct{h},\vct{x}_i\rangle\right)-\sigma\left(\langle \widehat{\vct{w}},\vct{x}_i\rangle\right)\right)}{s^2}\left(\vct{x}_i^T\vct{h}\right)}\\
\le&\frac{1}{s}\abs{\sigma'\left(\langle \widehat{\vct{w}},\vct{x}_i\rangle+s\langle \vct{h},\vct{x}_i\rangle\right)}\abs{\vct{x}_i^T\vct{h}}^2+\frac{\abs{\sigma\left(\langle \widehat{\vct{w}},\vct{x}_i\rangle+s\langle \vct{h},\vct{x}_i\rangle\right)-\sigma\left(\langle \widehat{\vct{w}},\vct{x}_i\rangle\right)}}{s^2}\abs{\vct{x}_i^T\vct{h}}\\
\le&\frac{2}{s}\abs{\vct{x}_i^T\vct{h}}^2\\
\le&\frac{6}{\epsilon}\abs{\vct{x}_i^T\vct{h}}^2 \;.
\end{align*}
Thus, by the mean value theorem we have
\begin{align*}
\abs{f(s)-f(\widetilde{s})}=&\abs{f'(ts+(1-t)s)) (s-\widetilde{s})}\\
\le&\frac{6}{\epsilon}\abs{\vct{x}_i^T\vct{h}}^2\abs{ s-\widetilde{s}} \;.
\end{align*}
\end{proof}
Applying the above lemma, we have
\begin{align}
\label{tmpsdev}
\abs{\mathcal{Z}(\vct{h};s)-\mathcal{Z}(\vct{h};\widetilde{s})}=&\frac{6}{\epsilon}\abs{ s-\widetilde{s}}\left(\frac{1}{m}\sum_{i=1}^m \abs{\vct{x}_i^T\vct{h}}^2\right)\nonumber\\
=&\frac{6}{\epsilon}\abs{ s-\widetilde{s}}\frac{\twonorm{\mtx{X}\vct{h}}^2}{m}\nonumber\\
\le &\frac{24}{\epsilon}\abs{ s-\widetilde{s}} \;.
\end{align}
Now let us consider an $\mathcal{N}_{\eta}$-cover of the $\frac{\epsilon}{3}\le s\le 4W$ with $\eta:=\frac{\xi\mu}{72}$. 
Using the union bound combined with \eqref{AC1} for any $\widetilde{s}\in \mathcal{N}_\eta$ we have $\mathcal{Z}(\vct{\vct{h}};\widetilde{s})\ge \left(1-\frac{2\xi}{3}\right)\mu$ holds for all $\vct{h}\in\mathbb{S}^{d-1}$ with probability at least 
\begin{align*}
1-\frac{4W}{\eta} e^{-m\frac{\mu^2}{10\gamma^2} \xi^2}=1-\left(\frac{288W}{\mu\xi}\right)e^{-m\frac{\mu^2}{10\gamma^2} \xi^2}=1-e^{\log\left(\frac{288W}{\mu\xi}\right)-m\frac{\mu^2}{10\gamma^2} \xi^2}\ge 1- e^{-m\frac{\mu^2}{11\gamma^2} \xi^2} \;,
\end{align*}
as long as $m \ge 110 \frac{\gamma^2}{\mu^2} \frac{\log\left(\frac{288W}{\mu\xi}\right)}{\xi^2}$. 
Therefore, using  \eqref{tmpsdev} for all $\vct{h}\in\mathbb{S}^{d-1}$ and all $\frac{\epsilon}{3}\le s\le 4W$ 
there exists $\widetilde{s}\in\mathcal{N}_\eta$ with $\abs{s-\widetilde{s}}\le \eta:=\frac{\xi\mu}{72}$. Thus,
\begin{align*}
\mathcal{Z}(\vct{h};s)\ge \mathcal{Z}(\vct{h};\widetilde{s})-\abs{\mathcal{Z}(\vct{h};s)-\mathcal{Z}(\vct{h};\widetilde{s})}\ge \left(1-\xi\right)\mu \;.
\end{align*}
In conclusion, for all $\vct{h}\in\mathbb{S}^{d-1}$ and all $\frac{\epsilon}{3}\le s\le CW$ we have that
\begin{align*}
\mathcal{Z}(\vct{h};s)\ge \left(1-\xi\right)\mu
\end{align*}
holds with probability at least $1- e^{-m\frac{\mu^2}{11\gamma^2} \xi^2}-e^{-c\sqrt{d}}$, 
as long as $m\ge 110 \frac{\gamma^2}{\mu^2} \frac{\log\left(\frac{288(W+1)}{\mu\xi}\right)}{\xi^2}$; 
which in turn with $\xi=\frac{1}{3}$ implies that
\begin{align}
\label{cor1}
\langle \nabla \widehat{L}^\sur(\vct{w})-\nabla \widehat{L}^\sur(\widehat{\vct{w}}),\vct{w}-\widehat{\vct{w}}\rangle \ge \frac{2}{3}\mu\twonorm{\vct{w}-\widehat{\vct{w}}}^2
\end{align} 
holds for all $\vct{w}$ obeying $\frac{\epsilon}{3}\le \twonorm{\vct{w}-\widehat{\vct{w}}}\le 4W$ with probability 
at least $1- e^{-m\frac{\mu^2}{99\gamma^2} }-e^{-c\sqrt{d}} \ge 1 - \delta$ as long as $m\gtrsim  \frac{\gamma^2}{\mu^2}d \log\left(\frac{(W+1)}{\mu \delta}\right)$. 

Now note that for any vector $\vct{w}\in\R^d$ and $\widehat{\vct{w}}$ we have
\begin{align}
\label{newcor1}
\langle \nabla \widehat{L}^\sur(\vct{w})-\nabla \widehat{L}^\sur(\widehat{\vct{w}}),\vct{w}-\widehat{\vct{w}}\rangle=& \frac{1}{m}\sum_{i=1}^m\left(\sigma\left(\langle \vct{w},\vct{x}_i\rangle\right)-\sigma\left(\langle \widehat{\vct{w}},\vct{x}_i\rangle\right)\right)\left(\vct{x}_i^T(\vct{w}-\widehat{\vct{w}})\right)\nonumber\\
\ge &\frac{1}{m}\sum_{i=1}^m \left(\sigma\left(\langle \vct{w},\vct{x}_i\rangle\right)-\sigma\left(\langle \widehat{\vct{w}},\vct{x}_i\rangle\right)\right)^2\nonumber\\
=&\frac{1}{m}\twonorm{\sigma(\mtx{X}\vct{w})-\sigma(\mtx{X}\widehat{\vct{w}})}^2 \;.
 \end{align}
 Also note that 
 \begin{align*}
 \twonorm{\nabla \widehat{L}^\sur(\vct{w})-\nabla \widehat{L}^\sur(\widehat{\vct{w}})}^2=&\frac{1}{m^2}\twonorm{\sum_{i=1}^m \left(\sigma(\vct{w}^T\vct{x}_i)-\sigma(\widehat{\vct{w}}^T\vct{x}_i)\right)\vct{x}_i}^2\\
 =&\frac{1}{m^2}\twonorm{\mtx{X}^T\left(\sigma(\mtx{X}\vct{w})-\sigma(\mtx{X}\widehat{\vct{w}})\right)}^2\\
 \le& \frac{\opnorm{\mtx{X}}^2}{m^2}\twonorm{\sigma(\mtx{X}\vct{w})-\sigma(\mtx{X}\widehat{\vct{w}})}^2\\
 \overset{(a)}{\le}&\frac{4}{m}\twonorm{\sigma(\mtx{X}\vct{w})-\sigma(\mtx{X}\widehat{\vct{w}})}^2\\
 \overset{(b)}{\le}& 4\langle \nabla \widehat{L}^\sur(\vct{w})-\nabla \widehat{L}^\sur(\widehat{\vct{w}}),\vct{w}-\widehat{\vct{w}}\rangle \;.
 \end{align*}
Here, (a) follows from the fact that under the log-concave density, centered and isotropy assumptions using \cite{adamczak2010quantitative} as long as $m\ge Cd$ we have that
\begin{align*}
\opnorm{\mtx{X}}\le 2\sqrt{m}
\end{align*}
holds with probability at least $1-e^{-c\sqrt{d}}$ and (b) follows from \eqref{newcor1}. Therefore,
\begin{align}
\label{cor2}
\langle \nabla \widehat{L}^\sur(\vct{w})-\nabla \widehat{L}^\sur(\widehat{\vct{w}}),\vct{w}-\widehat{\vct{w}}\rangle \ge \frac{1}{4}\twonorm{\nabla \widehat{L}^\sur(\vct{w})-\nabla \widehat{L}^\sur(\widehat{\vct{w}})}^2 \;.
\end{align} 
Combining \eqref{cor1} and \eqref{cor2}, we have that
\begin{align*}
\langle \nabla \widehat{L}^\sur(\vct{w})-\nabla \widehat{L}^\sur(\widehat{\vct{w}}),\vct{w}-\widehat{\vct{w}}\rangle =&\frac{1}{2}\langle \nabla \widehat{L}^\sur(\vct{w})-\nabla \widehat{L}^\sur(\widehat{\vct{w}}),\vct{w}-\widehat{\vct{w}}\rangle +\frac{1}{2}\langle \nabla \widehat{L}^\sur(\vct{w})-\nabla \widehat{L}^\sur(\widehat{\vct{w}}),\vct{w}-\widehat{\vct{w}}\rangle \\
\ge&\frac{\mu}{3}\twonorm{\vct{w}-\widehat{\vct{w}}}^2+ \frac{1}{8}\twonorm{\nabla \widehat{L}^\sur(\vct{w})-\nabla \widehat{L}^\sur(\widehat{\vct{w}})}^2 \;,
\end{align*} 
which gives us the desired result.

\section{Proof of Lemma~\ref{lem:concentration_additive}}\label{sec:concentration_additive}

Recall that $\nabla L^{\sur}_\D(\vct{w}) = \E[\sigma(\langle \vct{w}, \vct{x} \rangle)\vct{x}] - \chi_\D = \chi_\D^{\sigma_{\vct{w}}} - \chi_{\D}.$ For any fixed $\vct{w}$ observe that if $y$ is at most $1$, then for any $1$-Lipschitz and monotone function $\sigma$,  $(\sigma(\langle \vct{w}, \vct{x} \rangle) - y) \cdot \left \langle \bx, \frac{\chi_\D^{\sigma_{\vct{w}}} - \chi_{\D}}{\|\chi_\D^{\sigma_{\vct{w}}} - \chi_{\D}\|_2} \right \rangle $ is the product of random variables with tails bounded by $\exp(-\Omega(x^t))$. This implies
	\[
	\Pr_{S \sim \D^m} \left[  \left( \frac{1}{m}\sum_{i = 1}^m (\sigma(\langle \vct{w}, \bx^{(i)} \rangle ) - y^{(i)}) \cdot \left \langle \bx^{(i)}, \frac{\chi_\D^{\sigma_{\vct{w}}} - \chi_{\D}}{\|\chi_\D^{\sigma_{\vct{w}}} - \chi_{\D} \|_2} \right \rangle \right)  - \| \chi_\D^{\sigma_{\vct{w}}} - \chi_{\D}\|_2 \ge \epsilon  \|\chi_\D^{\sigma_{\vct{w}}} - \chi_{\D} \|_2 \right] \le \exp \left( -(\sqrt{m} \eps)^{t} \right). 
	\]
	Using the variational form of the norm (i.e., $\|\vct{v}\|_2:= \max_{\vct{u} \mid \|\vct{u}\|_2= 1} \iprod{\vct{u}, \vct{v}}$) on $ \left( \frac{1}{m}\sum_{i = 1}^m (\relu(\langle \vct{w}, \bx^{(i)} \rangle ) - y^{(i)}) \cdot \bx^{(i)} \right)  - (\chi_{\vct{w}} - \chi)$, we see that for any fixed $\vct{w} \in B(d,W)$
	\[
	\Pr_{S \sim \D^m} \left[  \left \| \left( \frac{1}{m}\sum_{i = 1}^m (\sigma(\langle \vct{w}, \x^{(i)} \rangle ) - y^{(i)}) \cdot  \x^{(i)}  \right)  - (\chi_\D^{\sigma_{\vct{w}}} - \chi_{\D}) \right \|_2\ge \epsilon  W \right] \le \exp \left( -(\sqrt{m} \eps)^{t} \right) \;. 
	\]
	
	%
	Taking a union bound over a $\gamma$-net $N_{\gamma}$ for $B(d, W)$ gives us 
	\[ \Pr_{\D} \left[ \forall \vct{w} \in N_{\gamma} \mid    \left \| \left( \frac{1}{m}\sum_{i = 1}^m (\sigma(\langle \vct{w}, \x^{(i)} \rangle ) - y^{(i)}) \cdot  \x^{(i)}  \right)  - (\chi_\D^{\sigma_{\vct{w}}} - \chi_{\D}) \right \|_2\ge \epsilon  W \right] \leq \exp \left( -(\sqrt{m} \eps)^{t} \right)  \cdot \left( \frac{3W}{\gamma} \right)^d  \;, \]
	i.e., for some constant $C$ depending on the distribution, we get
	\begin{align*} \Pr_{\D} \left[ \forall \vct{w} \in B(d, W) \mid    \left \| \left( \frac{1}{m}\sum_{i = 1}^m (\sigma(\langle \vct{w}, \x^{(i)} \rangle ) - y^{(i)}) \cdot  \x^{(i)}  \right)  - (\chi_\D^{\sigma_{\vct{w}}} - \chi_{\D}) \right \|_2\ge \epsilon  W + C \gamma \right] \\\leq \exp \left( -(\sqrt{m} \eps)^{t} \right) \cdot \left( \frac{3W}{\gamma} \right)^d  \;.
	\end{align*}	
	%
	%
	Hence, rescaling $\gamma$ we see that when $m \geq \Omega \left( \left( \frac{d}{\eps} \log \frac{W}{\gamma} \log \frac{1}{\delta} \right)^{2/t} \right)$ we have with probability $1-\delta$
	\[ \left \| \left( \frac{1}{m}\sum_{i = 1}^m (\sigma(\langle \vct{w}, \x^{(i)} \rangle ) - y^{(i)}) \cdot  \x^{(i)}  \right)  - (\chi_\D^{\sigma_{\vct{w}}} - \chi_{\D}) \right \|_2\leq \epsilon  W+ \gamma \;.
	\]
	Substituting $\gamma = \epsilon/2$ and rescaling $\epsilon$ we get the lemma.

	
	
	
	

\section{Proof of Theorem \ref{thm:nosc}}\label{sec:nosc}
	Let $\vct{w}^\sur$ be the minimizer of $L^\sur_\D$, then we have, for all $t$
	\begin{align*}
	&\|\vct{w}^{(t+1)} - \vct{w}^\sur\|_2^2 \\
	&\le \|\vct{v}^{(t+1)} - \vct{w}^\sur\|_2^2\\
	&= \|\vct{w}^{(t)} - \vct{w}^\sur\|_2^2 - \eta \langle \nabla\hat{L}^\sur_S(\vct{w}^{(t)}), \vct{w}^{(t)} - \vct{w}^\sur \rangle + \eta^2 \|\nabla\hat{L}^\sur_S(\vct{w}^{(t)})\|_2^2\\
	&\le \|\vct{w}^{(t)} - \vct{w}^\sur\|_2^2 - \eta \langle \nabla L^\sur_\D(\vct{w}^{(t)}), \vct{w}^{(t)} - \vct{w}^\sur \rangle - \eta \langle \nabla\hat{L}^\sur_S(\vct{w}^{(t)}) - \nabla L^\sur_\D(\vct{w}^{(t)}), \vct{w}^{(t)} - \vct{w}^\sur \rangle \\
	&\quad + 2\eta^2 \|\nabla L^\sur_\D(\vct{w}^{(t)})\|_2^2 + 2\eta^2 \|\nabla L^\sur_\D(\vct{w}^{(t)}) - \nabla\hat{L}^\sur_S(\vct{w}^{(t)})\|_2^2\\
&\le \|\vct{w}^{(t)} - \vct{w}^\sur\|_2^2 - \eta (1 - 2 \eta)\|_2\nabla L^\sur_\D(\vct{w}^{(t)})\|_2^2 + 2\eta \epsilon W + 2\eta^2 \epsilon^2 \;.
	\end{align*}
	For the final inequality, we use the smoothness and strong convexity of $L_\D^\sur$ to see 
	\[ L^\sur_\D(\vct{w}^\sur) \geq L^\sur_\D(\vct{w}^t) - \langle \nabla L^\sur_\D(\vct{w}^{(t)}), \vct{w}^{(t)} - \vct{w}^\sur \rangle + \frac{\mu}{2} \|w-w'\|_2^2\] 
	and 
	\[ -\| \nabla L^\sur_\D(\vct{w}^\sur) \|_2^2 \leq L^\sur_\D(\vct{w}^\sur) -  L^\sur_\D(\vct{w}^t). \]
	Additionally, for $\eta < 1/4$, either $\|\nabla L^\sur_\D(\vct{w}^{(t)})\|_2^2 \le 4\frac{\epsilon W + \eta \epsilon^2}{1 - 2 \eta}$ or $\|\vct{w}^{(t+1)} - \vct{w}^\sur\|_2^2 \le  \|\vct{w}^{(t)} - \vct{w}^\sur\|_2^2 - 2\eta \epsilon W - 2\eta^2\epsilon^2$. 
Therefore, after $T \geq \frac{2 W^2}{\eta \epsilon W +\eta^2\epsilon^2}$ iterations, there must exist some $t \le T$ such that $\|\nabla L^\sur_\D(\vct{w}^{(t)})\|_2^2 \le 4\frac{\epsilon W + \eta \epsilon^2}{1 - 2 \eta} \le 8\epsilon W + 2\epsilon^2$. Scaling $\epsilon$ appropriately gives us the result.

	\section{Proof of Lemma \ref{lem:sccl}} \label{sec:sccl}
By Definition \ref{def:sc}, we have that for all $\vct{u}, \vct{v}$,
\begin{align*}
&\langle \chi_\D^{\sigma_{\vct{u}}} - \chi_\D^{\sigma_{\vct{v}}}, \vct{u} - \vct{v} \rangle \ge \mu \| \vct{u} - \vct{v}\|_2^2 \implies \| \chi_\D^{\sigma_{\vct{u}}} - \chi_\D^{\sigma_{\vct{v}}}\|_2^2  \ge \mu^2 \| \vct{u} - \vct{v}\|_2^2.
\end{align*}
Also by 1-Lipschitzness of $\sigma$ and isotropy of $\D_\X$, we have 
\[
L_\D(\sigma_{\vct{u}}, \sigma_{\vct{v}}) \le \E[\langle\vct{u} - \vct{v}, \vct{x}\rangle^2] =  \| \vct{v} - \vct{u}\|_2^2.
\]
Combining the above gives us the desired result.

\section{Proof of Lemma \ref{lem:relusc}}\label{sec:relusc}
If $\vct{u} = \vct{v}$, then the claim follows directly. Let $\vct{u} \ne \vct{v}$. In the calculation below, we use monotonicity as well as 1-Lipschitzness. 
		\begin{align*}
		(\chi_\D^{\relu_{\vct{v}}} - \chi_\D^{\relu_{\vct{u}}} )^T(\vct{v} - \vct{u})
		&= \E\left[\left(\relu(\langle \vct{v}, \vct{x} \rangle) -\relu(\langle \vct{u}, \vct{x} \rangle)\right) \iprod{\vct{v} - \vct{u}, x}\right]\\
		&\ge \E\left[\left(\relu(\langle \vct{v}, \vct{x} \rangle) -\relu(\langle \vct{u}, \vct{x} \rangle)\right)^2 \right].
		\end{align*}
		The above term depends only on $\langle \vct{u}, \vct{x} \rangle$ and $\langle \vct{v}, \vct{x} \rangle$. To bound this, it is sufficient to work with $d = 2$. Let $f(x_1, x_2)$ be the density function for the log-concave distribution. The above term can be bounded based on the following two cases,
		\begin{itemize}
			\item\textbf{Case 1 }($\theta(\vct{u},\vct{v}) \le \pi/2$)
			We have,
			\begin{align*}
			(\chi_\D^{\relu_{\vct{v}}} - \chi_\D^{\relu_{\vct{u}}} )^T(\vct{v} - \vct{u}) &\ge\E\left[((\vct{v} - \vct{u}) \cdot x)^2 \mathbbm{1}[\langle \vct{u}, \vct{x} \rangle \ge 0, \langle \vct{v}, \vct{x} \rangle \ge 0] \right]\\
			&\ge\|\vct{v} - \vct{u}\|_2^2\E\left[(\langle \overline{\vct{v} - \vct{u}}, \vct{x} \rangle)^2 \mathbbm{1}[\langle \bar{\vct{u}}, \vct{x} \rangle \ge 0, \langle \bar{\vct{v}}, \vct{x} \rangle \ge 0] \right]\\
			&= \|\vct{v} - \vct{u}\|_2^2\int_{\R^2}(\langle \overline{\vct{v} - \vct{u}}, \vct{x} \rangle)^2 \mathbbm{1}[\langle \bar{\vct{u}}, \vct{x} \rangle \ge 0, \langle \bar{\vct{v}}, \vct{x} \rangle \ge 0] f(x_1, x_2)dx_1 dx_2 \\
			&\ge c\|\vct{v} - \vct{u}\|_2^2\int_{\|\vct{x}\|_2\le 1/9}\langle\overline{\vct{v} - \vct{u}}, \vct{x} \rangle^2 \mathbbm{1}[\langle \bar{\vct{u}}, \vct{x} \rangle \ge 0, \langle \bar{\vct{v}}, \vct{x} \rangle \ge 0] dx_1 dx_2 \;.
			\end{align*}
		Here the last inequality follows from the anti-concentration of the log-concave distribution. To prove strong convexity of the surrogate loss, it is sufficient to bound from below the above integral by a constant. 
Since the angle between $\vct{u}$ and $\vct{v}$ is less than $\frac{\pi}{2}$, 
we see that with respect to the uniform measure the set 
$\{ \vct{x} \mid \langle \vct{x}, \bar{\vct{u}} \rangle > 0  \text{ and } \langle \vct{x}, \bar{\vct{v}} \rangle > 0 \}$ 
has mass $> \frac{\vol(B(1/9))}{4}$. Thus, 
			\begin{align*}
			\int_{\|\vct{x}\|_2\le 1/9}(\langle \overline{\vct{v} - \vct{u}}, \vct{x} \rangle)^2 \mathbbm{1}[\langle \bar{\vct{u}}, \vct{x} \rangle \ge 0, \langle \bar{\vct{v}}, \vct{x} \rangle \ge 0] dx_1 dx_2 \ge \min\limits_{\substack{\vct{w}: \|\vct{w}\|_2= 1\\ S \subseteq B(1/9): \vol(S) = \frac{\vol(B(1/9))}{4}}} \int_{x \in S}\langle \vct{w}, \vct{x}\rangle^2 dx_1 dx_2 \;.
			\end{align*}
		To bound from below the integral above, let $\vct{x} = \alpha \vct{w} + \beta \vct{w}^{\perp}$ -- here we abuse notation slightly to and use $\vct{w}$ to denote the minimizer of the integral above. Since we are in 2 dimensions, the set that minimizes the integral is given by the region that minimizes $\iprod{\vct{w}, \vct{x}}$, subject to the volume constraint. Using the fact that in a ball of radius $1/9$ the distribution is lower bounded by a log-concave distribution, we see that it is sufficient to lower bound the integral for the above set  
$\{\alpha \vct{w} + \beta \vct{w}^{\perp} \mid |\alpha| < \gamma, \alpha^2 + \beta^2 \leq \frac{1}{9} \}$, 
for some constant $\gamma$. 		
		\begin{align*}
		&\int_{\|\vct{x}\|_2\le 1/9}(\langle \overline{\vct{v} - \vct{u}}, \vct{x} \rangle)^2 \mathbbm{1}[\langle \bar{\vct{u}}, \vct{x} \rangle \ge 0, \langle \bar{\vct{v}}, \vct{x} \rangle \ge 0] dx_1 dx_2  \\
		&\ge \min_{S \mid \Pr[S] = \Pr[\langle \vct{u}, \vct{x} \rangle \geq 0, \langle \vct{v}, \vct{x} \rangle \geq 0]} \int_{\alpha^2 + \beta^2 < 1/9} \alpha^2 \mathbbm{1}[S] d\alpha d \beta\\
		&\ge \int_{\alpha^2 + \beta^2 < 1/9, |\alpha| < \gamma} \alpha^2 d\alpha d \beta \\
		&\ge  \int_{ \alpha^2 + \beta^2 < 1/9, \frac{\gamma}{2} < \alpha < \gamma} \alpha^2 d\alpha d \beta \\
		&\ge  \int_{ \alpha^2 + \beta^2 < 1/9, \frac{\gamma}{2} < \alpha < \gamma} \frac{\gamma^2}{4} d\alpha d \beta \\
		&\ge \frac{\gamma^2}{4} \cdot c' \;.
		\end{align*}
Here $c'$ is a constant satisfying $c'= \vol(\{ (\alpha, \beta) \mid \alpha^2 + \beta^2 < 1/9, \alpha \in \left[ \gamma/2, \gamma \right]\})$.
		
\item \textbf{Case 2 }($\theta(\vct{u},\vct{v}) > \pi/2$): 
			We assume w.l.o.g. that $\|\vct{u}\|_2\ge \|\vct{v}\|_2$. Similar to the previous case, we have that
			\begin{align*}
			(\chi_\D^{\relu_{\vct{v}}} - \chi_\D^{\relu_{\vct{u}}} )^T(\vct{v} - \vct{u}) &\ge\E\left[\langle \vct{u}, \vct{x} \rangle^2 \mathbbm{1}[\langle \vct{u}, \vct{x} \rangle \ge 0, \langle \vct{v}, \vct{x} \rangle \le 0] \right]\\
			&\ge\|\vct{u}\|_2^2\E\left[\iprod{\bar{\vct{u}}, \vct{x}}^2 \mathbbm{1}[\langle \bar{\vct{u}}, \vct{x} \rangle \ge 0, \langle \bar{\vct{v}}, \vct{x} \rangle \le 0] \right]\\
			&\ge c\|\vct{u}\|_2^2\int_{\|\vct{x}\|_2\le 1/9}\iprod{\bar{\vct{u}}, \vct{x}}^2 \mathbbm{1}[\langle \bar{\vct{u}}, \vct{x} \rangle \ge 0, \langle \bar{\vct{v}}, \vct{x} \rangle \le 0] dx_1 dx_2\\
			&\ge \frac{c}{2}\|\vct{u} - \vct{v}\|_2^2\int_{\|\vct{x}\|_2\le 1/9}\iprod{\bar{\vct{u}}, \vct{x}}^2 \mathbbm{1}[\langle \bar{\vct{u}}, \vct{x} \rangle \ge 0, \langle \bar{\vct{v}}, \vct{x} \rangle \le 0] dx_1 dx_2 \;.
			\end{align*}
Since the angle between $\vct{u}$ and $\vct{v}$ is more than $\frac{\pi}{2}$ we see that with respect to the uniform measure the set $\{ \vct{x} \mid \langle \vct{x}, \bar{\vct{u}} \rangle > 0  \text{ and } \langle \vct{x}, \bar{\vct{v}} \rangle < 0 \}$ has mass $> \frac{\vol(B(1/9))}{4}$. The final integral above can again be lower bounded as in Case 1. 
\end{itemize}

\section{Proofs of Lemmas used for Theorem~\ref{thm:ptas}}

\begin{lemma}\label{lem:norm_to_angle}
	Suppose $\| \bw - \bw^*\|_2\leq O\left( \frac{\sqrt{\opt}}{\nu} \right)$ and let $\theta(\bw, \bw^*)$ represent the angle between $\bw$ and $\bw^*$ with $\|\bw^*\|_2= 1$. If $\theta(\bw, \bw^*) \le \pi/2$ then
	\[
	\frac{\theta(\bw,\bw^*)}{2} \le \sin(\theta(\bw, \bw^*)) \le O\left( \frac{\sqrt{\opt}}{\nu} \right) \;.
	\]
\end{lemma} 
\begin{proof}
	Since $\|\bw - \bw^*\|_2\le O\left( \frac{\sqrt{\opt}}{\nu} \right)$,
	\begin{align*}
	O(\opt/\nu^2) 	&\ge \|\bw - \bw^*\|_2^2 \\
	&= \|\bw\|_2^2 + \|\bw^*\|_2^2 - 2 \|\bw\|_2 \|\bw^*\|_2\cos(\theta(\bw, \bw^*))\\
	&= \|\bw\|_2^2 + 1 - 2 \|\bw\|_2 \cos(\theta(\bw,\bw^*)) \;.
	\end{align*}
	This implies
	\[
	\cos(\theta(\bw, \bw^*)) \ge \frac{\|\bw\|_2^2 + 1 - O(\opt/\nu^2)}{2 \|\bw\|_2} \ge \sqrt{1 - O(\opt/\nu^2)}.
	\]
	i.e., $\sin(\theta(\bw,\bw^*)) \le O\left( \frac{\sqrt{\opt}}{\nu} \right)$. Since $\theta(\bw, \bw^*) \in [0, \pi/2)$ we have $\frac{\theta(\bw, \bw^*)}{2} \le \sin(\theta(\bw, \bw^*))$.
	
\end{proof}

\begin{lemma}\label{lem:expectation_bound}
If $\bz$ is drawn from a $\nu$-subgaussian distribution, 
$\bz \in A := \{ \bx \mid \iprod{\vct{x},\vct{w}} \geq \gamma \sqrt{\opt} \text{ and } \iprod{\bw^*, \bx} \leq 0 \}$ 
and $\|\vct{w} - \vct{w}^*\|_2\leq O(\sqrt{\opt}/\nu)$, then 
\[ 
\E_{\bz \sim D}\left [\left (\frac{\Theta(1)}{\nu}\sqrt{\opt}\|\bz\|_2^2 + 
2\| \bz\|_2 \right ) 1_{A}(\bz)\right] \leq \nu \cdot \sqrt{\opt} \cdot \frac{\eta}{10} \;. \] 
\end{lemma}
\begin{proof}
	Since $\bz \in A$,  
	\[
	\gamma \sqrt{\opt}\le (\bw - \bw^*) \cdot \bz \le \|\bw - \bw^*\|_2 \|\bz\|_2\lesssim \frac{1}{\nu}\sqrt{\opt}\|\bz\|_2.
	\]
	Hence, $\|\bz\|_2\geq \Omega(\nu \gamma)$. 
To bound from above the expectation in question, we integrate in polar coordinates. Specifically,
we can write
	\begin{align*}
	\E_{\D}\left[\left(\frac{\Theta(1)}{\nu}\sqrt{\opt}\|\bz\|_2^2 + 2\|\bz\|_2\right) 1_{T_+ \land \overline{S}}(\bz)\right] &\lesssim \frac{1}{\nu} \int_{0}^{2\pi} \int_{0}^{\infty} \left(\frac{\Theta(1)}{\nu}\sqrt{\opt}r^3 + 2r^2\right) 1_{T_+ \land \overline{S}}(r, \theta) \exp(-r^2/2 \nu^2) dr d\theta \\
	&\lesssim  \frac{\theta(\bw, \bw^*)}{\nu} \int_{\Omega( \nu \gamma)}^{\infty} \left(\frac{\Theta(1)}{\nu}\sqrt{\opt}r^3 + 2r^2 \right)  \exp(-r^2/2\nu^2) dr \\
	&\lesssim  \theta(\bw, \bw^*)   \int_{\Omega(\gamma)}^{\infty} \left(\frac{1}{\nu}\sqrt{\opt}(\nu s)^3 + (\nu s)^2 \right)  \exp(-s^2/2) ds\\
	&= \nu^2 \theta(\bw, \bw^*)   \int_{\Omega(\gamma)}^{\infty} (\sqrt{\opt}s^3 + s^2)  \exp(-s^2/2) ds \;,
	\end{align*}
	where the final inequality is a consequence of $r = s \nu$ and $\gamma \geq 0$. 
We now use the following facts about the gaussian integrals of a $x^2$ and $x^3$ to bound from above the previous integral:
	\begin{equation}\label{eqn:gauss_2}
	\int_{t}^{\infty} x^2 \exp\left(-\frac{x^2}{2} \right)~dx \leq O\left( {\sf erfc}\left(\frac{t}{\sqrt{2}} \right)  + 2 t \exp(-t^2/2) \right) \leq  O\left( 1 + t \right) \cdot \exp(-t^2/2)
	\end{equation}
	\begin{equation}\label{eqn:gauss_3}
	\int_{t}^{\infty} x^3 \exp\left(-\frac{x^2}{2} \right)~dx  =  (2 + t^2) \cdot \exp(-t^2/2)
	\end{equation}
	The final inequality in Equation~\eqref{eqn:gauss_2} follows from  ${\sf erfc}(x) \leq 2 \exp(-x^2/2)$. 
	Making these substitutions and using Lemma~\ref{lem:norm_to_angle} to get the bound 
	$\theta(\bw, \bw^*) \leq O\left( \frac{\sqrt{\opt}}{\nu} \right)$, we see
	\begin{align*}
	\E_{\D}[((c/\nu)\sqrt{\opt}\|\bz\|_2^2 + 2\|\bz\|_2) 1_{T_+ \land \overline{S}}(\bz)]&\leq  O \left( \theta(\bw, \bw^*) \sqrt{\opt} \cdot \nu^2 \cdot (2+\gamma^2) \cdot \exp(-\gamma^2/2) \right)\\
	&+ O\left( \theta(\bw, \bw^*)  \cdot \nu^2 \cdot (1+\gamma) \cdot \exp(-\gamma^2/2) \right)\\
	&\lesssim  \nu \sqrt{\opt} \cdot \max \{ (1+\gamma^2), (1+\gamma) \} \exp( -\gamma^2 / 2) \;.
	\end{align*}
The choice $\gamma = \Omega \left( \sqrt{ \log \left(  \frac{1}{\eta} \right) }\right)$ ensures that this is at most 
$\nu \cdot \sqrt{\opt} \cdot \frac{\eta}{10}$.	
\end{proof}

\begin{lemma}[Probability of being in the band] \label{lem:prob_T} 
If $\D$ is $\nu$-subgaussian, and $\bw$ is the minimizer of the surrogate loss
	\[ \Pr_\D[|\iprod{\bw, \bx}| \le \gamma \sqrt{\opt}] \le \frac{c'\gamma \sqrt{\opt}}{\nu} \;. \]
\end{lemma} 
\begin{proof}
By standard properties of $\mathcal{N}(0, \nu I )$, we get the result for $\|\bw\|_2= 1$. 
However, we know that $\|\bw\|_2\ge \|\bw^*\|_2- \sqrt{\eta} = 1 - \sqrt{\eta}$ and 
we know that $1 - \sqrt{\eta}$ is larger than a constant which gives us the desired result.
\end{proof}

\subsection{Polynomial Approximation in the Band}
We show that there exists a low-degree polynomial approximation of the ReLU in squared error over the band.

\begin{lemma}\label{lem:poly_approx}
	If $\D$ is $\nu$-subgaussian, then there exists a degree $O(\frac{1}{\eta^3})$ polynomial $P$ satisfying
	\begin{equation}
	\E_{\D|_{T_{d, \gamma}}} \left[ (P(\iprod{\bw^*, \bx}) - \relu(\iprod{ \bw^*, \bx } ))^2 \right] \lesssim \eta^2 \cdot \nu \cdot \sqrt{\opt}
	\end{equation} 
	Here, $T_{d, \gamma}(\bw) = \{ \bu \in \mathbb{R}^d \mid |\iprod{\bw, \bu}| \leq \gamma \sqrt{\opt} \}$, where 
	$\gamma =  \sqrt{\log(1/\eta)}$. 
	
\end{lemma}
\begin{proof}
	An application of Jackson's theorem implies there is a degree $\frac{6s}{\tau}$ polynomial such that 
	\[ \|P(t) - \relu(t)\|_{[-s, s], \infty} \leq \tau \;.\] 
We apply the following Theorem~\ref{thm:sherstov} from~\cite{Sherstov-2013} to see that $P(t)$ satisfies 
	\[ | P(t) - \relu(t) | < 2 (4t/s)^{\frac{6s}{\tau}} \textit{ for } t \in \mathbb{R} \setminus [-s,s] \;.\]
	\begin{lemma}[Sherstov]\label{thm:sherstov}
		Let $p(t) := \sum_{i=1}^d a_i t^i$ be a given polynomial, then 
		\[ \sum_i |a_i| \leq 4^d \max_{i=0, \dots, d} \left| p\left( \frac{d - 2j}{d}\right) \right| \;.\]
	\end{lemma} 
	
	To estimate $\E_{\D|_{T}} \left[ (P(\iprod{\bw^*, \bx}) - \relu(\iprod{ \bw^*, \bx } ))^2 \right]$, we bound from above 
	the density function of $\iprod{\bw^*, \bx}$ where $\bx \sim \D_{T}$ (denoted by $\rho_{\gamma, \delta}$). 
	
	Since we only need to look at the distribution $\iprod{\bw^*, \bx}$, 
	it is sufficient to project the distribution $\D$ to $\text{span}(\bw, \bw^*)$. Henceforth, 
	we will abuse notation and let $\D$ refer to this projected distribution.
	
	 Let $\x \sim \D_T$ and $\overline{\bw} = \bw/\|\bw\|_2$. 
	 Express $\x$ in the basis $\overline{\bw}, \overline{\bw}^{\perp}$ 
	 to get $\x = \alpha_{\x} \overline{\bw} + \beta_{\bx} \overline{\bw}^{\perp}$. 
	 We now study the random variable $\iprod{\bw^*, \bx}$. 
	 Let $\theta = \theta(\bw, \bw^*)$ and note that $\|\bw^*\|_2=1$. 
	 For any $\bx \in T_{\gamma, d}(\bw)$. Hence,
	\begin{align*}
	\iprod{ \bw^*, \bx }  &= \alpha_{\bx} \iprod{\bw^*, \overline{\bw}} + \beta_{\bx}\iprod{\bw^*, \overline{\bw}^{\perp}}\\
&= \alpha_{\bx} \cos(\theta) + \beta_{\bx}\sin(\theta) \;,
	\end{align*}
	where ${\bw^*}^{\perp} = \bw^* - \overline{\bw} ( \overline{\bw} \cdot \bw^*)$ and $\|{\bw^*}^{\perp} \|_2= \sin(\theta)$. 
	We now individually upper bound the pdf of $\alpha_{\bx} \cos(\theta)$ and 
	$\beta_{\bx}\iprod{{\bw^*}^\perp, \overline{\bw}^{\perp}}$ respectively, when $\bx \sim \D_T$. 
	Recall that $\D$ is $\nu$-subgaussian, hence the pdf of $\alpha_{\bx}$, $\rho_{\alpha_{\bx}}$ 
	is upper bounded by $\frac{1}{\nu} \exp(-x^2/2\nu^2)$. Similarly, for $\beta_{\bx}$. 
	If $f_{a,b}$ is the pdf of the joint distribution over $(a,b) \in \R^2$, then the pdf of $a+b$ is given by 
	$\rho_{a+b}(z) = \int_{-\infty}^{\infty} f_{a,b}(t,z-t)dt$. For $\bx \sim \D_T$, the pdf of $\bx$ 
	in the basis above (i.e., the pdf of $(\alpha_{\bx}, \beta_{\bx})$) is upper bounded by 
	$\rho_\D(\bx) \leq \frac{1}{\Pr_\D[T]}\frac{100}{\nu^2} \cdot \exp(-(\alpha_{\bx}^2 + \beta_{\bx}^2)/2\nu^2)$, 
	when $\alpha_{\bx} \leq \gamma \sqrt{\opt} \|w\|_2$, and $0$ otherwise. 
	Since $\iprod{{\bw^*},{\bx}} = \alpha_{\bx} \cos(\theta) + \beta_{\bx}\sin(\theta)$, the pdf of 
	$\iprod{ \bw^*, \bx }$, $\rho_{\iprod{\bw^*,\bx}}$ is bounded by 
	\[ \rho_{\iprodtwo{\bw^*}{\bx}}(z) \lesssim \frac{1}{\Pr[T]} \int_{-b}^b \frac{1}{\nu^2 \sin(\theta) \cos(\theta)} \cdot \exp(-(t^2/ \cos^2(\theta) + (z-t)^2 /\sin^2(\theta))/2\nu^2)dt \;. \]
%
	We bound the above expression in two ways. First, since $\exp(-(z-t)^2 /2\sin^2(\theta)\nu^2) \leq \exp(-t^2 /2\sin^2(\theta)\nu^2) \leq 1$, we can write
	\begin{align*}
	\rho_{\iprod{ \bw^*, \bx }}(z) &\lesssim \frac{1}{\Pr[T]} \int_{-b}^b \frac{1}{\nu^2 \sin(\theta) \cos(\theta)} \cdot \exp(-(t^2/ \cos^2(\theta) + (z-t)^2 /\sin^2(\theta))/2\nu^2)dt \\ 
	&\lesssim \frac{1}{\Pr[T]} \int_{-b}^b \frac{1}{\nu^2 \sin(\theta) \cos(\theta)} \cdot \exp(-(t^2/ \cos^2(\theta) + t^2 /\sin^2(\theta))/2\nu^2)dt \\
	&\lesssim \frac{1}{\Pr[T] \nu \sin(\theta)} \int_{-b}^b \frac{1}{\nu \cos(\theta)} \cdot \exp(-(t^2/ 2\nu^2\cos^2(\theta)))dt \\
	&\lesssim \frac{1}{\nu \sin(\theta)} \;.
	\end{align*}	
	Also, if for some $l$ it holds $|z| \geq l$, we have that
	\begin{align*}
	\rho_{\iprod{ \bw^*, \bx }}(z) &\lesssim \frac{1}{\Pr[T]} \int_{-b}^b \frac{1}{\nu^2 \sin(\theta) \cos(\theta)} \cdot \exp(-(t^2/ \cos^2(\theta) + (z-t)^2 /\sin^2(\theta))/2\nu^2)dt \\ 
	&\lesssim \frac{1}{\Pr[T]} \int_{-b}^b \frac{1}{\nu^2 \sin(\theta) \cos(\theta)} \cdot \exp(-(t^2/ \cos^2(\theta) + (z-l)^2 /\sin^2(\theta))/2\nu^2)dt \\
	&\lesssim \exp(-(z-l)^2/2\nu^2\sin^2(\theta)) \cdot (1/\nu \sin(\theta)) \;.
	\end{align*}
We will use the above bound with $l = \gamma \sqrt{\opt} \| \bw \|_2$. 
We now use the above upper bounds to bound the total error in the band
	\begin{align*}
	\E_{\D|_{T}} \left[ (P(\iprod{\bw^*, \bx}) - \relu(\iprod{ \bw^*, \bx } ))^2 \right] &\leq  \int_{-s}^{s} (P(t) - \relu(t))^2 \rho_{\iprod{\bw^*, \bx}}(t) dt + 2\int_{s}^{\infty} (P(t) - \relu(t))^2 \rho_{\iprod{\bw^*, \bx}}(t)  dt \;.
	\end{align*}
	Recall that $P$ is at most $\tau$-away from $\relu$ in the range $[-s, s]$. If $\tau = \nu^{1/2} \opt^{1/4} \cdot (\eta / 10)$, 
	the first integral above can be upper bounded as follows:
	\begin{align*} 
	\int_{-s}^{s} (P(t) - \relu(t))^2 \rho_{\iprod{\bw^*, \bx}}(t)dt &\leq \int_{-s}^{s} \tau^2 \rho_{\iprod{\bw^*, \bx}}(t)dt \lesssim \nu \cdot \eta^2 \cdot \opt^{1/2} \;.
	\end{align*}
	Additionally, if $t \geq 2l$, $(t-l)^2 \leq t^2/4$, and so under the condition $s \geq 2l$, 
	the final integral can be bounded above by
	\begin{align*}
	2\int_{s}^{\infty} (P(t) - \relu(t))^2 \rho_{\iprod{\bw^*, \bx}a}(t)  dt &\lesssim \frac{1}{\nu \cdot \sin(\theta)} \int_{s}^{\infty}  \max \{(4t/s)^{6s/\tau}, (4t)^{s/\tau} \}\exp \left( -\frac{(t - l)^2}{2 \cdot \nu^2 \cdot \sin^2(\theta)}\right) dt\\
	&\lesssim \frac{1}{ \nu \cdot \sin(\theta)} \int_{s}^{\infty} (4t/s)^{6s/\tau} \exp \left( -\frac{t^2}{8 \cdot \nu^2 \cdot \sin^2(\theta)}\right) dt \;.
	\end{align*}
	Setting $t = 2\nu \sin(\theta) p$ and $r = s/(\nu \sin(\theta))$, we see 
	\begin{align*}
	&2\int_{s}^{\infty} (P(t) - \relu(t))^2 \rho_{\iprod{\bw^*, \bx}a}(t)  dt \\
	&\lesssim  \int_{s/(\nu \sin(\theta))}^{\infty}  (4p \cdot (\nu \sin(\theta)/s))^{6s/\tau} \exp \left( -\frac{p^2}{8}\right) dp\\
	&\lesssim  \int_{r}^{\infty}  (4p \cdot (1/r))^{6s/\tau} \exp \left( -\frac{p^2}{8}\right) dp\\
	&\lesssim \E_{p \sim N(0,1)}[1_{p \geq r} (4p/r)^{3s/\tau}] \\ 
	&\lesssim  (4/r)^{3s/\tau} \cdot \Pr[p \geq r]^{1/2} \cdot \E_{p \sim N(0,1)}[ p^{6s/\tau} ]^{1/2}\\
	&\lesssim  (4/r)^{3s/\tau} \cdot  \exp(-r^2/4) \cdot (6s/\tau)^{3s/\tau}\\
	&\lesssim s/(r^2 \tau))^{3s/\tau} \cdot  \exp(-r^2/4) \;.
	\end{align*}	
	Taking logs, we see that the following inequality needs to be satisfied 	
	\begin{align}
 \left( \frac{s}{\tau} +  \frac{s}{\tau} \log \left( \frac{s}{\tau} \right) + \frac{s}{\tau} \log \left( \frac{1}{r} \right) - \log (\nu^{1/2} \eta \opt^{1/4} )\right) \lesssim  r^2 \;.
	\end{align}	
	Let $s = \nu (\nu \sin(\theta))^{1/2} / (\nu^{1/2}  \eta^2 \opt^{1/4})$, then $r = \nu^{1/2} \eta^{-2} (\nu \sin(\theta))^{-1/2} \opt^{-1/4}$. Then, $s/\tau \leq f(\eta)$ for $f(x) = 1/x^3$. Substituting for $s/\tau$ we see that it is sufficient to check 
	\begin{align}
	\left( f(\eta) +  f(\eta) \log \left( f(\eta) \right) + f(\eta) \log \left(\nu^{-1/2} \eta^{2} (\nu \sin(\theta))^{1/2} \opt^{1/4} \right) - \log (\nu^{1/2} \eta \opt^{1/4} )\right) \lesssim r^2 \;.
	\end{align}
	Since $f(\eta) \geq 1$ and $\nu \sin(\theta) \leq \opt^{1/2}$, it is sufficient to check that 	
	\begin{align}
	\left( f(\eta) +  f(\eta) \log \left( \eta^2 f(\eta) \right) + (f(\eta)-1) \log \left( (\nu \sin(\theta))^{1/2} \opt^{1/4} \right) - \log (\eta) + f(\eta) \log(\nu^{-1/2}) - \log(\nu^{1/2}) \right) \lesssim r^2 \;.
	\end{align}
 Substituting for $f(\eta)$ and multiplying both sides by $\eta^{-3}$, we get 
	\begin{align}
\left( \log \left( 1/\eta\right) + \log \left( \nu^{-1/2} (\nu \sin(\theta))^{1/2} \opt^{1/4} \right) \right) \lesssim \nu \eta^{-1} \cdot (\nu \sin(\theta))^{-1} \cdot \opt^{-1/2} \;,
	\end{align}
	i.e., rescaling, we see 
	\[ \log(r^{-2} \eta^{-3}) \lesssim r^2 \eta^3 \;. \]
Since $r \geq \nu^{1/2} \eta^{-2} \sin(\theta)^{-1} \nu^{-1} \geq \eta^{-2} \nu^{-1/2} > 1$ for $\opt, \nu$ less than some constant, 
we see that this is true for small enough $\eta$. 

%

Substituting this back we see that the overall error is bounded by 
\begin{align}
\E_{\D|_{T}} \left[ (P(\bx) - \relu(\iprod{ \bw^*, \bx } ))^2 \right] &\lesssim \nu \cdot \eta \cdot \sqrt{\opt} \;.
\end{align}
This completes the proof.
\end{proof}

\begin{lemma}\label{lem:concentration2}
	Let 
	\[ A := \left \{ p \mid \text{$p$ is a degree $k$, $d$-variate polynomial with coefficients $a_i$ such that $\sum_i |a_i| \leq O(4^k)$}  \right \} \;. \]  
	If $S \sim \D^m$ is a set of $m$ iid samples and $m \gtrsim  k^k \cdot \frac{1}{\eps^2} \cdot \left( \frac{4d}{\nu^2} \right)^{\Omega(k)}$, then 
	\[
	\E_{S|_T}[(P(\bx) - y)^2] \le \min_{P'\in A}\E_{\D|_T}[(P'(\bx) - y)^2] + \epsilon \;.
	\]
\end{lemma}
\begin{proof}
	Let $p(\vct{x}) = \sum a_i m_i$, where $\{ m_i \}$ are monomials that correspond to $a_i$. Then
	\begin{align*} 
	(\E_{S|_T}[P(\bx)] - \E_{\D|_T}[P(\bx)]) &\leq \left( \sum_i |a_i| \right) \left( \sum_i |\E_{S|_T}[m_i(\bx)] - \E_{S|_T}[m_i(\bx)]| \right) \\
	&\leq \left(4\right)^k \left( \sum_i |\E_{S|_T}[m_i(\bx)] - \E_{S|_T}[m_i(\bx)]| \right) \;.
	\end{align*}
Since there are at most $d^k$ different monomials, it is sufficient to approximate each monomial $m_i$ up 
to an accuracy of $\eps \cdot \frac{1}{4^k} \cdot \frac{1}{d^k}$. 
	
	Let $M_{d,k}$ denote the set of $d$-variate monomials of degree $k$. Applying Fact~\ref{lem:concentration} on these monomials and noting that the variance of any monomial is at most $\nu^{2k} \cdot O(k)^k$ with respect to the $\nu$-subgaussian distribution where $\nu = O(1)$, we get 
	\[ 
	\Pr \left[ m \in M_{d, k} \mid \E_{S|_T}[m(\bx)] - \E_{\D|_T}[m(\bx)] | \geq \eps \right] \leq d^k \cdot \exp \left(-\frac{m\eps^2 \cdot \frac{1}{4^k} \cdot \frac{1}{d^k}}{\nu^{2k} k^k} \right)^{1/k} \;.
	\]
	Setting $m \gtrsim O(k)^k \cdot \frac{1}{\eps^2} \cdot \left( \frac{4d}{\nu^2} \right)^{O(k)}$ we obtain our desired bounds. 
\end{proof}


\end{document}